\documentclass{opendatalab}

\usepackage[utf8]{inputenc}
\usepackage{xcolor}
\usepackage[most]{tcolorbox}
\usepackage{longtable}
\usepackage{listings}
\usepackage[numbers]{natbib}
\usepackage{enumitem}
\usepackage{graphicx}
\usepackage{xspace} 
\usepackage{hyperref}
\usepackage{multirow}
\usepackage{adjustbox}
\usepackage{algorithm}
\usepackage{algorithmic}
\usepackage{amsfonts}
\usepackage{mathrsfs}
\usepackage{subcaption}
\usepackage{wrapfig}
\usepackage{amsmath}
\usepackage{amsthm}

\newtheorem{theorem}{Theorem}
\newtheorem{remark}{Remark}

\newtheorem{lemma}{Lemma}
\newtheorem{proposition}{Proposition}
\newtheorem{corollary}{Corollary}

\definecolor{codegreen}{rgb}{0,0.6,0}
\definecolor{codegray}{rgb}{0.5,0.5,0.5}
\definecolor{codepurple}{rgb}{0.58,0,0.82}
\definecolor{backcolour}{rgb}{0.95,0.95,0.92}
\definecolor{promptcolor}{HTML}{D1D0F2}
\definecolor{promptcolorheader}{HTML}{bdbcec}

\definecolor{promptcolor}{HTML}{E3F0FA}
\definecolor{promptcolorheader}{HTML}{B5D6ED}
\definecolor{prompttitletext}{HTML}{1B3A5C}

\newtcolorbox{promptbox}[1][]{
	enhanced, breakable,
	top=0.3em,bottom=0.3em,left=0.5em,right=0.5em,
	toptitle=0.3em,bottomtitle=0.2em,boxsep=0pt,
	colframe=promptcolorheader, colback=promptcolor!50, boxrule=0.5pt,
	width=\columnwidth, 
	coltitle=prompttitletext,
	title={\footnotesize #1} 
}
\lstdefinestyle{promptstyle}{
    backgroundcolor=\color{backcolour},   
    commentstyle=\color{codegreen},
    keywordstyle=\color{magenta},
    numberstyle=\tiny\color{codegray},
    stringstyle=\color{codepurple},
    basicstyle=\ttfamily\footnotesize,
    breakatwhitespace=false,         
    breaklines=true,                 
    captionpos=b,                    
    keepspaces=true,                 
    numbers=left,                    
    numbersep=5pt,                  
    showspaces=false,                
    showstringspaces=false,
    showtabs=false,                  
    tabsize=2
}
\title{Offline Policy Optimization with Posterior Sampling}
\author[1]{Hongqiang Lin}
\author[2]{Dongxu Zhang}
\author[2]{Yiding Sun}
\author[1]{Mingzhe Li}
\author[3]{Ning Yang}
\author[4]{Haijun Zhang}

\affiliation[1]{Zhejiang University}
\affiliation[2]{Xi’an Jiaotong University}
\affiliation[3]{Institute of Automation, Chinese Academy of Sciences}
\affiliation[4]{University of Science and Technology Beijing}

\abstract{
A fundamental challenge in model-based offline reinforcement learning (RL) lies in the trade-off between generalization and robustness against exploitation errors in out-of-distribution (OOD) regions. The key to resolving this trade-off lies in enabling the model to explore OOD regions that remain consistent with underlying physical dynamics. However, achieving this is challenging because limited data cannot uniquely identify the dynamics model, and unconstrained exploration is risky. Existing methods often overlook this nuance, addressing the risk through excessive pessimistic regularization, which ensures robustness but sacrifices generalization. To address this, we propose PSPO, which treats the dynamics model as a random variable rather than a point estimate. This formulation inherently allows for controlled exploration of OOD regions. By alternately updating the posterior distribution and the policy, we design a regularized optimization algorithm with convergence guarantees. Experiments on standard benchmarks demonstrate that PSPO achieves superior performance compared to state-of-the-art baselines. Further analysis confirms that our method attains the desired pessimism-free property while maintaining robustness, and ablation studies verify the effectiveness of each proposed module.
}

\date{\today}
\correspondence{Ning Yang, \email{ning.yang@ia.ac.cn}
\\
\\
}

\begin{document}

\maketitle

% \tableofcontents

\section{Introduction}
Offline reinforcement learning (RL) \cite{Levine2020OfflineRL,moerland2023model} learns optimal policies from fixed datasets, avoiding the safety risks associated with online trial-and-error. The primary challenge in offline RL is distribution shift \cite{adaptive2023zhang,yang2025rtdiff}, where the learned policy deviates from the support of the behavior policy during policy iteration. This deviation exposes the policy to out-of-distribution (OOD) regions. Unconstrained policy iteration in these OOD regions tends to exploit approximation errors, causing erroneous action selection and performance collapse \cite{imagination2025liu}. Consequently, offline RL involves an inherent trade-off between behavioral regularization and value maximization \cite{park2025flow}.

In model-based offline RL, this trade-off manifests through the utilization of synthetic transitions. On the one hand, synthesizing dynamics-consistent OOD transitions promotes value maximization by facilitating generalization beyond the offline dataset support. On the other hand, it simultaneously introduces risks of model exploitation \cite{kidambi2020morel,lurevisiting,lin2025anystep}. Achieving this trade-off is challenging because limited data cannot uniquely identify the dynamics model, and unconstrained exploration is unsafe. Existing approaches typically mitigate this risk through pessimism, penalizing OOD value estimates based on epistemic uncertainty (e.g., the standard deviation of model predictions) to restrict the policy within the dataset support \cite{yu2020mopo,sun2023model,Qiao_Lyu_Jiao_Liu_Li_2025}. However, these pessimistic methods disproportionately favor behavioral regularization to ensure safety, inevitably sacrificing generalization. 

Thus, there exists a natural question to be addressed: \textit{Can we design a model-based offline RL algorithm that balances generalization and robustness against model exploitation in OOD regions?}

In this paper, we affirmatively answer this question by proposing Posterior Sampling-based Policy Optimization (PSPO). Grounded in the Bayesian RL framework \cite{ghavamzadeh2015bayesian,ijcai2024p427}, PSPO treats dynamics model as a random variable rather than a single point estimate. This probabilistic formulation captures epistemic uncertainty through a posterior distribution over candidate models, allowing the policy to adaptively balance generalization and robustness against model exploitation. 

To analyze policy optimization process, we decompose it into two alternating stages: updating the critic and posterior with the policy fixed, and optimizing the policy objective with the posterior fixed. A key challenge arises from the mutual dependence between these stages. Establishing convergence thus requires controlling this feedback loop across full cycles. We bound the propagation of critic and policy variations through posterior reweighting and derive a sufficient stability condition. Under this condition, solving both stages exactly yields joint convergence to a consistent critic, posterior, and policy.

\begin{wrapfigure}{r}{0.5\textwidth}
	\vspace{-25pt}
	\begin{minipage}{\linewidth}
		\begin{figure}[H]
			\centering
			\includegraphics[width=\textwidth]{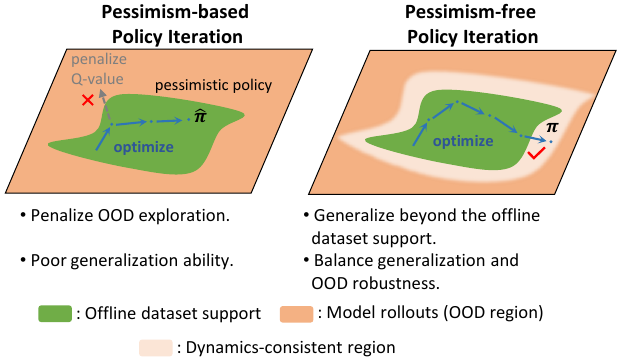}
			\caption{(Left) Existing pessimism-based policy iteration often results in overly pessimistic policies. (Right) We propose a pessimism-free paradigm designed to strike a balance between generalization to dynamics-consistent regions and robustness against model exploitation in OOD regions.}
			\label{fig1:intro}
		\end{figure}
	\end{minipage}
	\vspace{-20pt}
\end{wrapfigure}

Our main contributions are summarized as follows:
\paragraph{Algorithm.}
PSPO formulates dynamics modeling as a Bayesian inference process. Unlike prior works that rely on fixed posterior assumptions, we derive the posterior $P(T|E)$ conditioned on a batch of state-action queries $E$, which weights learned models according to critic-dependent TD consistency. By performing model rollouts via posterior sampling, PSPO allows the policy to explore dynamics-consistent OOD regions. By evaluating the value of these OOD transitions, PSPO facilitates generalization beyond the offline dataset support. Building on this, we develop a constrained policy optimization algorithm that limits policy changes and ensures robustness against OOD exploitation. Notably, PSPO establishes a pessimism-free paradigm, distinguishing itself from prior model-based offline RL approaches (Figure~\ref{fig1:intro}).
\paragraph{Theoretical Analysis.}
We establish Bellman contraction and regularized optimality under a fixed posterior, and prove that solving the regularized problem improves the original objective within a local neighborhood. Building on this fixed-posterior analysis, we derive sufficient conditions for the convergence of alternating posterior reconstruction and policy optimization. These conditions ensure convergence to a unique, mutually consistent triple of posterior, critic, and policy.
\paragraph{Experimental Results.}
Extensive evaluations across deterministic and stochastic benchmarks demonstrate that PSPO achieves state-of-the-art (SoTA) performance. Finally, ablation studies verify the indispensability of each key module within the algorithm.

\section{Related Work}
\paragraph{Offline RL}
Offline RL aims to learn optimal policies entirely from a fixed dataset of historical interactions, without further access to the environment \cite{Levine2020OfflineRL}. The primary challenge in offline RL is distribution shift \cite{brandfonbrener2021offline,adaptive2023zhang}. In practice, this issue is typically mitigated by employing pessimistic approaches. In model-free offline RL, pessimism is realized either by explicitly penalizing the Q-values of OOD actions \cite{kumar2020conservative,yeom2024exclusively,huang2024efficient}, or by implicitly assigning a low weight to them \cite{kostrikov2021offline,maodoubly}. In model-based offline RL, pessimism is commonly introduced by either constructing a dynamics model ensemble and using the predicted uncertainty (e.g., standard deviation) as a penalty \cite{kidambi2020morel,yu2020mopo,sun2023model,Qiao_Lyu_Jiao_Liu_Li_2025}, or by formulating a two-player zero-sum game based on robust MDPs \cite{rigter2022rambo,guo2022model,dong2024online}. However, recent literature confirms a positive correlation between the generalization ability of offline RL and state-action space coverage, suggesting that pessimism may detrimentally affect generalization \cite{mediratta2024gengap,park2024value}.
Therefore, despite efforts to broaden coverage via optimistic sampling \cite{zhai2024optimistic} or high-fidelity trajectory synthesis \cite{lin2025anystep}, these approaches still resort to uncertainty-based penalization to realize pessimism, limiting their ability to fully leverage model generalization. We propose a novel model-based algorithm that treats the dynamics model as a random variable and maintains a corresponding distribution. Compared to prior methods, our approach explicitly avoids uncertainty-based penalization, leading to an empirically non-pessimistic framework.
\paragraph{Bayesian RL}
Bayesian RL manages uncertainty by first constructing a posterior belief over an ensemble of dynamics models and subsequently optimizing a policy under this belief \cite{wei2023bayesian,NEURIPS2024_fdb11be1,ni2025long,pmlr-v202-sasso23a}. Existing Bayesian RL works diverge primarily in belief construction and policy optimization. For belief construction, existing methods either derive the posterior through Bayesian inference with an implicit likelihood or assume a fixed posterior directly induced by the offline dataset \cite{rigter2022rambo,guo2022model,dong2024online,ijcai2024p427}. Policy optimization in Bayesian RL typically introduces pessimism via risk-sensitive objectives (e.g., CVaR)~\cite{lobo2020soft,rigter2022planning} or regularized policy iteration~\cite{neu2017unifiedviewentropyregularizedmarkov,guo2022model,lin2026robust}. In the offline setting, where finite data cannot uniquely identify the true dynamics~\cite{ghosh2022offline}, prior works similarly enforce conservative constraints on belief formation or optimization~\cite{rigter2022rambo,ghosh2022offline,rigter2023one}. Such reliance on pessimism inevitably limits the inherent adaptability of Bayesian RL. In contrast, PSPO imposes no structural assumptions on the posterior form and avoids pessimism. Our approach achieves generalization through posterior sampling, while ensuring robustness against OOD exploitation through the policy optimization process.

\section{Preliminaries and Notations}
\label{sec:preliminaries}
We consider a discounted MDP with state space $\mathcal{S}$, action space $\mathcal{A}$, bounded reward $r:\mathcal{S}\times\mathcal{A}\to[-R_{\max},R_{\max}]$, initial distribution $\rho_0$, and discount factor $\gamma\in(0,1)$. Let $T:\mathcal{S}\times\mathcal{A}\to\Delta(\mathcal{S})$ denote the transition kernel, where $\Delta(\mathcal{S})$ is the space of probability distributions over $\mathcal{S}$. For a policy $\pi:\mathcal{S}\to\Delta(\mathcal{A})$, the return is defined as $\mathcal{J}(\pi,T)=\mathbb{E}_{\rho_0,\pi,T}[\sum_{t\ge0}\gamma^t r(s_t,a_t)]$, with corresponding action-value and value functions $Q_T^\pi(s,a)=\mathbb{E}_{\pi,T}[\sum_{t\ge0}\gamma^t r(s_t,a_t)\mid s_0=s,a_0=a]$ and $V_T^\pi(s)=\mathbb{E}_{a\sim\pi(\cdot|s)}[Q_T^\pi(s,a)]$. During learning, $Q_k$ denotes the current critic estimate, with $V_k(s)=\mathbb{E}_{a\sim\pi_k(\cdot|s)}[Q_k(s,a)]$; these estimates are shared across candidate models.

Bayesian RL represents uncertainty about the dynamics through a belief over transition models \cite{ghosh2022offline}. Given a prior $P_0$ and observed data $E$, Bayesian inference gives $P(T|E)\propto P_0(T)P(E|T)$. A Bayes-adaptive Markov decision process (BAMDP) augments the physical state $s$ with the current belief $P$. Its policy can depend on the augmented state $(s,P)$, and the belief is updated as new transitions are observed. In practice, maintaining an exact posterior belief over high-dimensional dynamics is intractable. Deep ensembles offer a scalable approximation to this belief~\cite{chua2018deep,rigter2022rambo,guo2022model,rigter2023one,dong2024online}. We define the candidate set as the support of the prior, $\mathcal{T} = \operatorname{supp}(P_0) = \{T_m : P_0(T_m) > 0\}$, with a uniform prior. While $\mathcal{T}$ encompasses transition models, the offline dataset only covers a limited set of state-action pairs. Consequently, model predictions may become unreliable when extrapolating beyond this data support.

\section{Formulation}
In this section, we first define the return induced by posterior-sampled transitions, then construct posterior distribution. This separates the choice of synthetic dynamics from the policy regularization introduced in the next section.
\subsection{Objective Function}
\label{sec:objective}
The optimization objective for PSPO is to maximize the expected return averaged over the posterior distribution of the model dynamics:
\begin{equation}
	\label{pre:eq4}
	\mathcal{J}(\pi)=\mathop {\mathbb E}_{\substack{\rho_0, \pi\\
			{T_0\sim P(T|E)}}} \left[\mathop {\mathbb E}\limits_{\substack{T_0,\pi \\ T_1\sim P(T|E)}}\left[\cdots \mathop {\mathbb E}_{T_{\infty}, \pi}\big[\sum_{t=0}^\infty\gamma^tr(s_t,a_t)\big]\right]\right].
\end{equation} 
This formulation explicitly integrates epistemic uncertainty into value estimation by treating the transition dynamics as random variables drawn from the posterior $P(T|E)$.  Under the condition that a model is sampled from the posterior $P(T|E)$ at the beginning of each episode and held fixed for its duration, or if the posterior distribution is a point mass, then Equation~\eqref{pre:eq4} reduces to:
\begin{equation}
	\label{pre:eq5}
	\mathcal{J}(\pi)=\mathop{\mathbb{E}}_{\pi,T\sim P(T|E)}[\mathcal{J}(\pi,T)].
\end{equation}

The policy obtained by maximizing Equation~\eqref{pre:eq5} with respect to $\pi$ is known as the Bayes-optimal policy \cite{ghosh2022offline}. In this paper, we focus on the generalized objective of Bayesian RL defined by Equation~\eqref{pre:eq4}, with the proposed algorithm retaining full applicability to the objective in Equation~\eqref{pre:eq5}.

\subsection{Posterior Distribution \texorpdfstring{$P(T|E)$}{P(T|E)}}
\label{sec:posterior}
The model distribution determines which learned dynamics generate synthetic transitions \cite{guo2022model}. At posterior update $k$, fix the current $Q_k$ and $\pi_k$, and collect an evidence batch $E_k=\{(s_{j,k},a_{j,k})\}_{j=1}^{b}$ of state-action queries from the offline data or model rollouts. Rewards are evaluated using the shared function $r(s,a)$. For each candidate $T$, the Bellman target is $(\mathcal B_T^{\pi_k}Q_k)(s,a)=r(s,a)+\gamma\mathbb E_{s'\sim T(\cdot|s,a)}V_k(s')$. We define the consistency score as follows:
\begin{equation}
\label{reweight_metric}
\mathcal F_k(T;E_k)=\frac1b\sum_{j=1}^{b}\left[Q_k(s_{j,k},a_{j,k})-(\mathcal B_T^{\pi_k}Q_k)(s_{j,k},a_{j,k})\right]^2.
\end{equation}
Thus, $E_k$ defines the domain over which consistency is measured. Both $Q_k$ and $V_k$ remain fixed across candidate models. The next-state expectation is taken under each $T$, rather than treating $E_k$ as one observed transition. Equation~\eqref{reweight_metric} reweights the ensemble without refitting its dynamics parameters.

We balance this consistency score against deviation from the prior through
\begin{equation}
\label{model:eq1}
P_k=\arg\min_{P\in\Delta(\mathcal T)}\left\{D_{\mathrm{KL}}(P\Vert P_0)+\beta\mathbb E_{T\sim P}[\mathcal F_k(T;E_k)]\right\},
\end{equation}
where $\beta>0$. Its solution is $P_k(T)=P_0(T)\exp[-\beta\mathcal F_k(T;E_k)]/Z_k$, with $Z_k$ the normalizing constant. The prior $P_0$ remains fixed, while the consistency scores are recomputed prior to weight normalization. We use $P(T|E_k)$ as shorthand for this generalized posterior, whose dependence on $Q_k$ and $\pi_k$ is implicit. This solution interprets $P(E|T) \propto \exp(-\beta\cdot\mathcal{F}(T))$ as a generalized likelihood, aligning with the variational framework of generalized Bayesian inference \cite{bissiri2016general}.

\subsection{Theoretical Properties of Posterior Sampling}
\label{sec:model}
For the analysis, $\mathcal S$, $\mathcal A$, and $\mathcal T$ are finite, $|r|\le R_{\max}$, and the learned models are fixed. Let $Q_{\max}\ge R_{\max}/(1-\gamma)$ and $\mathcal Q=\{Q:\|Q\|_\infty\le Q_{\max}\}$. For a fixed evaluation policy $\pi$, write $V_Q^\pi(s)=\mathbb E_{a\sim\pi(\cdot|s)}Q(s,a)$. Averaging the model-specific targets gives
\begin{equation}
\label{model:operator}
(\mathcal B_P^\pi Q)(s,a)=r(s,a)+\gamma\mathbb E_{T\sim P,\,s'\sim T(\cdot|s,a),\,a'\sim\pi(\cdot|s')}Q(s',a').
\end{equation}
\begin{proposition}[Fixed-Posterior Policy Evaluation]
\label{proposition1}
For fixed $P$ and $\pi$, $\mathcal B_P^\pi$ maps $\mathcal Q$ into itself and is a $\gamma$-contraction in the sup norm. Its unique fixed point is $Q_P^\pi$.
\end{proposition}
\begin{proposition}[Conditional Sampling Variance]
\label{thm:var}
Let \(\mathscr H_t\) denote the information available before sampling the target at iteration \(t\). Let $Q_t\in\mathcal Q$, $\pi_t$, and $P_t:=P(T|E)$ be fixed. If $T_t\sim P_t$ and $Y_t(s,a)=(\mathcal B_{T_t}^{\pi_t}Q_t)(s,a)$, then $$\mathbb E[Y_t|\mathscr H_t]=\mathcal B_{P_t}^{\pi_t}Q_t,\quad \operatorname{Var}(Y_t(s,a)|\mathscr H_t)\le Q_{\max}^2.$$
\end{proposition}
\begin{remark}[Convergence of Q-Value Estimation]
\label{remark1}
Posterior sampling gives the update
\begin{equation}
\label{thm:policy_eval}
Q_{t+1}(s,a)=Q_t(s,a)+\eta_t\big(Y_t(s,a)-Q_t(s,a)\big).
\end{equation}
With bounded sampling noise, a contractive conditional-mean operator and suitable step sizes yield convergence \cite{robbins1951stchastioc}. Proposition~\ref{proposition1} supplies contraction for fixed $P_t$. 
\end{remark}
\begin{remark}[Connection to Thompson Sampling]
PSPO implements a variant of Thompson Sampling by drawing dynamics models from the posterior $P(T|E)$ to synthesize transitions. Unlike standard ensemble methods that often average predictions, our approach preserves the stochasticity inherent in the posterior variance to  translate epistemic uncertainty into controlled exploration. Furthermore, by modulating the posterior via the consistency score $\mathcal{F}(T)$ defined in Equation~\eqref{reweight_metric}, the algorithm adaptively concentrates probability mass toward models with lower disagreement. This enables PSPO to selectively leverage dynamics-consistent OOD transitions for effective generalization beyond the offline dataset support, while suppressing hallucinations arising from poorly fitted models.
\end{remark}

\section{Posterior Sampling-based Policy Optimization (PSPO)}
Posterior reweighting governs the composition of synthetic experience, while policy optimization determines the behavioral changes learned from it. We integrate these two mechanisms through a regularized objective and a constrained policy-update algorithm, and subsequently analyze how posterior reconstruction and policy optimization interact across training cycles.
\subsection{Iterative Regularized Policy Optimization}
\label{sec:policy}
Optimizing the objective function $\mathcal{J}(\pi)$ (Equation~\eqref{pre:eq4}) without resorting to pessimism (e.g., uncertainty penalization of the Q-function) requires addressing two key challenges:
\begin{enumerate}
	\item Regulate OOD exploration to maintain adherence to the offline dataset support, effectively mitigating distribution shift while permitting safe exploration.
	\item Alleviate optimization instability stemming from the non-stationary state visitation distribution induced by time-varying dynamics and evolving policies.
\end{enumerate} 

To mitigate the first challenge, we augment the objective function with a regularization term that prevents the policy from excessively sampling OOD actions. Letting $\mu$ denote the reference policy and $\alpha > 0$ denote the regularization strength, we consider the following regularized objective:
\begin{equation}
\begin{aligned}
\label{opt:eq1}
\widehat{\mathcal{J}}_P^\alpha(\pi)=\mathop {\mathbb E}_{\substack{\rho_0, \pi\\
		{T_0\sim P(T|E)}}} \bigg[\mathop {\mathbb E}\limits_{\substack{T_0,\pi \\ T_1\sim P(T|E)}}&\big[\cdots \mathop {\mathbb E}_{T_{\infty}, \pi}\big[\sum_{t=0}^\infty\gamma^t\big(r(s_t,a_t)\\&-\alpha D_\mathrm{KL}(\pi(\cdot|s_t)||\mu(\cdot|s_t))\big)\big]\big]\bigg],
\end{aligned}
\end{equation}
The corresponding soft Bellman optimality operator is
\begin{equation}
\label{opt:eq2}
(\mathcal B_P^*Q)(s,a)=r(s,a)+\gamma\mathop{\mathbb E}_{\substack{T\sim P,\,\\{s'\sim T(\cdot|s,a)}}}\left[\alpha\log\mathbb E_{a'\sim\mu(\cdot|s')}e^{Q(s',a')/\alpha}\right].
\end{equation}
The soft state value in Equation~\eqref{opt:eq2} results from optimizing the next-state action distribution with a KL penalty relative to $\mu$. For a fixed state $s$, define
$Z_Q(s)=\mathbb E_{a\sim\mu(\cdot|s)}e^{Q(s,a)/\alpha}$ and $\pi_Q^*(a|s)=\mu(a|s)e^{Q(s,a)/\alpha}/Z_Q(s)$. For any policy $\pi(\cdot|s)$ supported by $\mu(\cdot|s)$,
\begin{equation}
\begin{aligned}
\mathbb E_{a\sim\pi(\cdot|s)}Q(s,a)
-&\alpha D_{\mathrm{KL}}\!\left(\pi(\cdot|s)\,\|\,\mu(\cdot|s)\right)
\\&=\alpha\log Z_Q(s)-\alpha D_{\mathrm{KL}}\!\left(\pi(\cdot|s)\,\|\,\pi_Q^*(\cdot|s)\right).
\end{aligned}
\end{equation}
Since the final KL divergence is nonnegative, the maximum is attained at $\pi_Q^*$ and equals $V_Q^{\mathrm{soft}}(s)=\alpha\log Z_Q(s)$. Substituting this optimal next-state value into the Bellman backup yields the log-expectation term in Equation~\eqref{opt:eq2}.
\begin{theorem}[Fixed-Posterior Regularized Optimality]
\label{thm2}
For fixed posterior $P$, $\mathcal B_P^*$ is a $\gamma$-contraction with a unique fixed point $Q_P^*$. Its Gibbs policy $\pi_P^*(a|s)\propto\mu(a|s)e^{Q_P^*(s,a)/\alpha}$ maximizes $\widehat{\mathcal{J}}_P^\alpha$.
\end{theorem}

We write $\mathcal{J}_P(\theta)=\mathcal{J}_P(\pi_\theta)$ and $\widehat{\mathcal{J}}_P^\alpha(\theta)=\widehat{\mathcal{J}}_P^\alpha(\pi_\theta)$ when using policy coordinates. For an incremental policy update, define $D_{\mathrm{KL},\nu}(\pi\Vert\pi_k)=\mathbb E_{s\sim\nu}D_{\mathrm{KL}}(\pi(\cdot|s)\Vert\pi_k(\cdot|s))$, where the fixed distribution $\nu$ has full support. The constrained update is
\begin{equation}
\label{opt:eq3}
\pi_{k+1}\in\arg\max_\pi \widehat{\mathcal{J}}_{P_k}^\alpha(\pi),\qquad \text{s.t.} \quad D_{\mathrm{KL},\nu}(\pi\Vert\pi_k)\le\epsilon.
\end{equation}

\begin{theorem}[Local Policy Improvement]
\label{thm3}
Fix $P_k$ and suppose $\pi_k$ has full support. In independent tabular-policy coordinates $\theta$, let $g_k=\nabla \widehat{\mathcal{J}}_{P_k}^\alpha(\theta_k)\ne0$ and let $F_k\succ0$ be the Hessian of the policy KL at $\theta_k$. Suppose both objectives are twice continuously differentiable locally and $\nabla \mathcal{J}_{P_k}(\theta_k)^TF_k^{-1}g_k>0$. For any exact solution $\theta_k+\delta_\epsilon$ of Equation~\eqref{opt:eq3}, as $\epsilon\rightarrow0$,
\begin{equation*}
\begin{aligned}
\mathcal{J}_{P_k}(\theta_k+\delta_\epsilon)-\mathcal{J}_{P_k}(\theta_k)
=\sqrt{\frac{2\epsilon}{g_k^TF_k^{-1}g_k}}
\nabla \mathcal{J}_{P_k}(\theta_k)^TF_k^{-1}g_k+o(\sqrt\epsilon).
\end{aligned}
\end{equation*}
Consequently, every exact solution of Equation~\eqref{opt:eq3} improves $\mathcal{J}_{P_k}$ for all sufficiently small $\epsilon>0$.
\end{theorem}
The local improvement conditions ensure that the regularized policy step also improves the unregularized return. 

\begin{corollary}[Local Improvement for Neural Policies]
\label{cor:neural}
Fix $P_k$ and let $\pi_\phi$ be a neural policy with $\pi_{\phi_k}=\pi_k$. Suppose the policy has full support locally, and $h_k(\phi)=D_{\mathrm{KL},\nu}(\pi_\phi\Vert\pi_k)$ is twice continuously differentiable near $\phi_k$. In these parameters, define $g_k=\nabla_\phi\widehat{\mathcal{J}}_{P_k}^\alpha(\pi_\phi)|_{\phi_k}\ne0$ and $F_k=\nabla_\phi^2h_k(\phi_k)\succ0$. Under the assumption that $\nabla_\phi\mathcal{J}_{P_k}(\pi_\phi)|_{\phi_k}^{T}F_k^{-1}g_k>0$. For any sufficiently small fixed $r>0$, consider the restriction of Equation~\eqref{opt:eq3} to policies $\pi_\phi$ satisfying $\|\phi-\phi_k\|_2\le r$. Every exact maximizer then satisfies the return expansion stated in Theorem~\ref{thm3} (with derivatives taken with respect to $\phi$) and improves $\mathcal{J}_{P_k}$ for all sufficiently small $\epsilon>0$.
\end{corollary}

\begin{remark}[Scope of Policy Improvement]
Let $\mathcal{J}_P(\pi)$ denote the original return in Equation~\eqref{pre:eq4} evaluated under posterior $P$. The return change across one policy update and posterior reconstruction decomposes as
\begin{equation*}
\begin{aligned}
&\mathcal{J}_{P_{k+1}}(\pi_{k+1})-\mathcal{J}_{P_k}(\pi_k)\\
&=\underbrace{
\mathcal{J}_{P_k}(\pi_{k+1})-\mathcal{J}_{P_k}(\pi_k)
}_{\text{policy update at fixed }P_k}\\
&\quad+\underbrace{
\mathcal{J}_{P_{k+1}}(\pi_{k+1})-\mathcal{J}_{P_k}(\pi_{k+1})
}_{\text{posterior update at fixed }\pi_{k+1}}.
\end{aligned}
\end{equation*}
The first term is positive under the local improvement conditions of Theorem~\ref{thm3}. The second term evaluates the same policy under two different posteriors and can have either sign. Monotonic improvement across outer iterations therefore requires any decrease from posterior reconstruction to be no larger than the gain from the policy update. With the posterior held fixed, repeated steps satisfying the theorem yield increasing, bounded return values and hence a convergent return sequence. This convergence of return values alone does not establish convergence to an optimal policy.
\end{remark}

\begin{remark}[Outer Reconstruction and Inner Optimization]
At outer iteration $k$, Equation~\eqref{model:eq1} reconstructs $P_k$ from the current critic and evidence batch. The inner phase holds $P_k$ fixed while updating the critic and policy through Equations~\eqref{opt:eq2}--\eqref{opt:eq3}. For the idealized soft critic loop, $Q_k^{(0)}=Q_k$, $Q_k^{(j+1)}=\mathcal B_{P_k}^*Q_k^{(j)}$, and $Q_{k+1}=Q_k^{(n)}$. Thus $n$ counts Bellman backups between posterior reconstructions.
\end{remark}

\subsection{Convergence of Alternating Posterior and Policy Updates}
\label{sec:adaptive}
We consider an alternating procedure that uses the soft Bellman optimality operator in Equation~\eqref{opt:eq2} and completely solves the regularized policy objective. Each cycle $k$ has two stages.

\paragraph{Stage 1: Update the Critic and Posterior with the Policy Fixed}
Hold $\pi_k$ fixed. Repeatedly reconstruct the posterior from the current critic, then update that critic using the reconstructed weights: $P\leftarrow P_{Q,\pi_k}$ and $Q\leftarrow\mathcal B_P^*Q$. Here $P_{Q,\pi}$ denotes the posterior in Equation~\eqref{model:eq1}, whose consistency score depends on both $Q$ and $\pi$ through Equation~\eqref{reweight_metric}. Under the conditions below, these iterations converge to $\bar Q_k,\bar P_k$ satisfying
\begin{equation*}
\bar P_k=P_{\bar Q_k,\pi_k},
\qquad
\bar Q_k=\mathcal B_{\bar P_k}^*\bar Q_k.
\end{equation*}

\paragraph{Stage 2: Optimize the Policy with the Posterior Fixed}
Hold $\bar P_k$ fixed and choose the Gibbs policy induced by $\bar Q_k$:
\begin{equation*}
\pi_{k+1}(a|s)=\pi_{\bar Q_k}(a|s)
=\frac{\mu(a|s)e^{\bar Q_k(s,a)/\alpha}}
{\sum_{a'}\mu(a'|s)e^{\bar Q_k(s,a')/\alpha}}.
\end{equation*}
By Theorem~\ref{thm2}, this policy maximizes $\widehat{\mathcal{J}}_{\bar P_k}^\alpha$. The new policy $\pi_{k+1}$ is then used to score models in Stage 1 of the next cycle. Thus the procedure alternates a converged critic--posterior solve with exact regularized policy optimization. We use $\pi_Q$ to denote the Gibbs policy induced by any critic $Q$.

\paragraph{Conditions for Convergence}
We use the finite-space, bounded-reward setting of Section~\ref{sec:model}. Let $\Pi=\prod_{s\in\mathcal S}\Delta(\mathcal A)$ denote the tabular policy space, equipped with $d(\pi,\pi')=\max_s\sum_a|\pi(a|s)-\pi'(a|s)|$. Each critic phase starts in $\mathcal Q=\{Q:\|Q\|_\infty\le Q_{\max}\}$, where $\|f\|_\infty=\max_{s,a}|f(s,a)|$.

Define the uniform residual and model-disagreement bounds
\begin{equation*}
\begin{aligned}
D&=\sup_{\pi\in\Pi,\,Q\in\mathcal Q}
\max_{T\in\mathcal T}\|Q-\mathcal B_T^\pi Q\|_\infty,\\
\Delta&=\sup_{Q\in\mathcal Q}
\max_{T,T'\in\mathcal T}
\|\mathcal B_T^*Q-\mathcal B_{T'}^*Q\|_\infty.
\end{aligned}
\end{equation*}
Here $D$ bounds the ordinary Bellman residual used in the consistency score $\mathcal F_{Q,\pi}(T;E)$, uniformly over critics and scoring policies. The bound $\Delta$ measures the largest difference between two models' soft targets at the same critic and state-action pair; $\mathcal B_T^*$ uses model $T$ alone.

Set $\lambda=\beta D\Delta(1+\gamma)/(1-\gamma)$ and $\kappa=\gamma+(1-\gamma)\lambda$. The following lemma establishes convergence within Stage 1 and bounds how its limiting critic changes with the scoring policy.

\begin{lemma}[Convergence Within a Fixed-Policy Critic Phase]
\label{lem:critic_phase}
Suppose $\lambda<1$. For each fixed scoring policy $\pi$, the map $\mathcal G_n^\pi(Q)=(\mathcal B_{P_{Q,\pi}}^*)^nQ$ is a contraction on $\mathcal Q$ with modulus at most $\gamma^n+(1-\gamma^n)\lambda$ for every integer $n\ge1$. Its limit as $n\to\infty$ has modulus at most $\lambda$. These maps share a unique fixed point $\bar Q(\pi)$, and their iterates converge to this point even with varying $n$. Moreover, for any $\pi,\pi'\in\Pi$,
\begin{equation*}
\|\bar Q(\pi)-\bar Q(\pi')\|_\infty
\le
\frac{\beta D\Delta\gamma Q_{\max}}
{(1-\gamma)(1-\lambda)}
\,d(\pi,\pi').
\end{equation*}
\end{lemma}

The case $n=1$ is the Stage 1 map $\mathcal A_\pi(Q)=\mathcal B_{P_{Q,\pi}}^*Q$, whose contraction modulus is at most $\kappa<1$. Larger $n$ allows more critic updates between posterior reconstructions. Consequently, $\bar Q_k=\bar Q(\pi_k)$ and one complete cycle is the policy update $\pi_k\mapsto\pi_{\bar Q(\pi_k)}$. The following theorem controls this complete cycle.

\begin{theorem}[Alternating Posterior and Policy Convergence]
\label{thm:adaptive}
Under the setting above, suppose
$
\frac{\beta D\Delta}{1-\gamma}
\left(1+\gamma+\frac{\gamma Q_{\max}}{\alpha}\right)<1.
$
Then the cycle map $\pi\mapsto\pi_{\bar Q(\pi)}$ is a contraction on $(\Pi,d)$ with modulus at most $\rho=\frac{\beta D\Delta\gamma Q_{\max}}{\alpha(1-\gamma)(1-\lambda)}<1$. From any $\pi_0\in\Pi$, the alternating procedure satisfies $d(\pi_k,\pi^\star)\le\rho^k d(\pi_0,\pi^\star)$, $\bar Q_k\to Q^\star$, and $\bar P_k\to P^\star$. Its limit is the unique triple in $\Pi\times\mathcal Q\times\Delta(\mathcal T)$ satisfying
\begin{equation*}
P^\star=P_{Q^\star,\pi^\star},
\qquad
Q^\star=\mathcal B_{P^\star}^*Q^\star,
\qquad
\pi^\star=\pi_{Q^\star}.
\end{equation*}
In particular, $\pi^\star$ maximizes $\widehat{\mathcal{J}}_{P^\star}^\alpha$.
\end{theorem}

\begin{corollary}[Neural Approximation of the Alternating Updates]
\label{cor:neural_alternating}
Under the assumptions of Theorem~\ref{thm:adaptive}, let $\pi_k=\pi_{\phi_k}$ and let $Q_{\theta_k}\in\mathcal Q$ be the neural critic returned by Stage 1. Reconstruct $P_k=P_{Q_{\theta_k},\pi_k}$ and define the critic and policy errors as $u_k=\|Q_{\theta_k}-\bar Q(\pi_k)\|_\infty$ and $v_k=d(\pi_{\phi_{k+1}},\pi_{Q_{\theta_k}})$. Then
\begin{equation*}
d(\pi_{k+1},\pi^\star)\le\rho\,d(\pi_k,\pi^\star)+u_k/\alpha+v_k.
\end{equation*}
If $u_k,v_k\to0$, then $\pi_k\to\pi^\star$ in $d$, $Q_{\theta_k}\to Q^\star$ in the sup norm, and $P_k\to P^\star$ in $\ell_1$. If instead $u_k/\alpha+v_k\le\bar e$ for all $k$, then $$\limsup_k d(\pi_k,\pi^\star)\le\bar e/(1-\rho).$$
\end{corollary}

\paragraph{Relationship Between the Guarantees}
Theorem~\ref{thm2} establishes regularized optimality for a fixed posterior and supplies the Gibbs policy used in Stage 2. Lemma~\ref{lem:critic_phase} establishes convergence of the critic and posterior within Stage 1. Theorem~\ref{thm:adaptive} combines the two stages and controls policy changes across complete cycles. Corollary~\ref{cor:neural_alternating} extends this conclusion to neural updates whose critic and policy errors vanish, and bounds the policy error when approximation errors persist. Its requirements concern uniform errors in the represented functions; a small training loss on sampled data alone does not establish them. These results retain the finite-space setting and concern network outputs rather than parameter convergence. Theorem~\ref{thm3} and Corollary~\ref{cor:neural} address a different policy step: they establish local improvement of the original return for exact KL-constrained updates at a fixed posterior. 

\subsection{Algorithm and Implementation}
Algorithm~\ref{algo} gives the practical neural implementation of PSPO. At each iteration, the posterior is reconstructed using the current critic and policy, and the resulting distribution over dynamics models guides subsequent critic and policy updates.
\begin{algorithm}[t!]
    \caption{Posterior Sampling-based Policy Optimization}
    \label{algo}
    \begin{algorithmic}[1]
        \REQUIRE Offline dataset $\mathcal D$, model ensemble size $N$, iteration steps $I$.
        \STATE \textbf{Initialization:} Randomly initialize the critic $Q_{\theta_0}(s,a)$ and policy $\pi_{\phi_0}(a|s)$. Set $\theta'_0\gets\theta_0$ and $\phi'_0\gets\phi_0$ to initialize the target networks $Q_{\theta'_0}$ and $\pi_{\phi'_0}$. Randomly initialize $N$ dynamics models $\{T_{\psi_i}(s'|s,a)\}_{i=1}^N$.
        \STATE \textbf{Dynamics model training:} Train each $T_{\psi_i}$ by maximizing $\mathbb E_{(s_t,a_t,s_{t+1})\sim\mathcal D}[\log T_{\psi_i}(s_{t+1}|s_t,a_t)]$.
        \FOR{$k=1,\ldots,I$}
            \STATE \textbf{Outer posterior reconstruction:} Form $E_k$, evaluate the consistency scores in Equation~\eqref{reweight_metric} using $Q_{\theta_{k-1}}$ and $\pi_{\phi_{k-1}}$, and normalize Equation~\eqref{model:eq1} to obtain $P_k$.
            \STATE \textbf{Inner critic update:} With $P_k$ fixed, use Equation~\eqref{opt:eq2} to update $Q_{\theta_{k-1}}$ and obtain $Q_{\theta_k}$.
            \STATE \textbf{Inner policy update:} With $P_k$ fixed, apply Equation~\eqref{opt:eq3} with current policy $\pi_{\phi_{k-1}}$ to obtain $\pi_{\phi_k}$.
            \STATE \textbf{Moving average:} Update target parameters $\theta'_k$ and $\phi'_k$ by moving averages of $(\theta'_{k-1},\theta_k)$ and $(\phi'_{k-1},\phi_k)$, respectively.
        \ENDFOR
        \RETURN $\pi_{\phi_I}$.
    \end{algorithmic}
\end{algorithm}
\begin{table*}[t!]
	\centering
	\caption{Performance comparison on D4RL datasets. Scores are normalized as $(\mathrm{score} - \mathrm{random}) / (\mathrm{expert} - \mathrm{random})$ and presented as mean $\pm$ standard deviation. Our results are averaged across 4 random seeds.}
	\label{tab:d4rl_performance}
	\setlength{\tabcolsep}{2pt} 
	\adjustbox{width=0.99\textwidth}
	{
		\begin{tabular}{lcccccccc}
			\toprule
			Dataset & CQL & DMG & EPQ & MOReL & RAMBO & PMDB & ADM & PSPO (Ours) \\
			\midrule
			% --- HalfCheetah ---
			HalfCheetah-Random        & 31.3$\pm$3.5 & 28.8$\pm$1.3 & 33.0$\pm$2.4 & 38.9$\pm$1.8 & 39.5$\pm$3.5 & 37.8$\pm$0.2 & \textbf{45.4}$\pm$2.8 & 37.7$\pm$0.5\\
			HalfCheetah-Medium        & 46.9$\pm$0.4 & 54.9$\pm$0.2 & 67.3$\pm$0.5 & 60.7$\pm$4.4 & 77.9$\pm$4.0 & 75.6$\pm$1.3 & 72.2$\pm$0.6 & \textbf{79.3}$\pm$0.9\\
			HalfCheetah-Expert        & 97.3$\pm$1.1 & 95.9$\pm$0.3 & \textbf{107.2}$\pm$0.2 & 8.4$\pm$11.8 & 79.3$\pm$15.1 & 105.7$\pm$1.0 & 89.4$\pm$26.4 & 93.3$\pm$0.6\\
			HalfCheetah-Medium-Expert & 95.0$\pm$1.4 & 91.1$\pm$4.2 & 95.7$\pm$0.3 & 80.4$\pm$11.7 & 95.4$\pm$5.4 & 108.5$\pm$0.5 & 103.7$\pm$0.2 & \textbf{109.7}$\pm$1.6\\
			HalfCheetah-Medium-Replay & 45.3$\pm$0.3 & 51.4$\pm$0.3 & 62.0$\pm$1.6 & 44.5$\pm$5.6 & 68.7$\pm$5.3 & 71.7$\pm$1.1 & 67.6$\pm$3.4 & \textbf{78.4}$\pm$0.8\\
			HalfCheetah-Full-Replay   & 76.9$\pm$0.9 & 79.9$\pm$1.2 & 85.3$\pm$0.7 & 70.1$\pm$5.1 & 87.0$\pm$3.2 & 90.0$\pm$0.8 & 86.3$\pm$1.7 & \textbf{94.9}$\pm$1.2\\
			\midrule
			% --- Hopper ---
			Hopper-Random        & 5.3$\pm$0.6  & 20.4$\pm$10.4 & 32.1$\pm$0.3 & \textbf{38.1}$\pm$10.1 & 25.4$\pm$7.5 & 32.7$\pm$0.1 & 32.7$\pm$0.2 & 31.9$\pm$0.6\\
			Hopper-Medium        & 61.9$\pm$6.4 & 100.6$\pm$1.9 & 101.3$\pm$0.2 & 84.0$\pm$17.0 & 87.0$\pm$15.4 & 106.8$\pm$0.2 & 107.4$\pm$0.6 & \textbf{108.5}$\pm$1.3\\
			Hopper-Expert        & 106.5$\pm$9.1 & 111.5$\pm$2.2 & 112.4$\pm$0.5 & 80.4$\pm$34.9 & 50.0$\pm$8.1 & 111.7$\pm$0.3 & 102.3$\pm$11.9 & \textbf{116.6}$\pm$1.2\\
			Hopper-Medium-Expert & 96.9$\pm$15.1 & 110.4$\pm$3.4 & 108.8$\pm$5.2 & 105.6$\pm$8.2 & 88.2$\pm$20.5 & 111.8$\pm$0.6 & 112.7$\pm$0.3 & \textbf{112.8}$\pm$1.4\\
			Hopper-Medium-Replay & 86.3$\pm$7.3 & 101.9$\pm$1.4 & 97.8$\pm$1.0 & 81.8$\pm$17.0 & 99.5$\pm$4.8 & 106.2$\pm$0.6 & 104.4$\pm$0.4 & 
			\textbf{110.0}$\pm$1.1\\
			Hopper-Full-Replay   & 101.9$\pm$0.6 & 106.4$\pm$1.1 & 108.5$\pm$0.6 & 94.4$\pm$20.5 & 105.2$\pm$2.1 & 109.1$\pm$0.2 & 108.5$\pm$0.7 & \textbf{112.5}$\pm$1.6\\
			\midrule
			% --- Walker2d ---
			Walker2d-Random        & 5.4$\pm$1.7  & 4.8$\pm$2.2  & \textbf{23.0}$\pm$0.7 & 16.0$\pm$7.7 & 0.0$\pm$0.3 & 21.8$\pm$0.1 & 22.2$\pm$0.2 & 22.1$\pm$0.1\\
			Walker2d-Medium        & 79.5$\pm$3.2 & 92.4$\pm$2.7 & 87.8$\pm$2.1 & 72.8$\pm$11.9 & 84.9$\pm$2.6 & 94.2$\pm$1.1 & 93.2$\pm$1.1 & \textbf{103.9}$\pm$2.5\\
			Walker2d-Expert        & 109.3$\pm$0.1 & \textbf{114.7}$\pm$0.4 &109.8$\pm$1.0 & 62.6$\pm$29.9 & 1.6$\pm$2.3 & 115.9$\pm$1.9 & 5.5$\pm$1.3 & 109.8$\pm$1.4\\
			Walker2d-Medium-Expert & 109.1$\pm$0.2 & 114.4$\pm$0.7 & 112.0$\pm$0.6 & 107.5$\pm$5.6 & 56.7$\pm$39.0 & 111.9$\pm$0.2 & 114.9$\pm$0.3 & \textbf{116.1}$\pm$1.0\\
			Walker2d-Medium-Replay & 76.8$\pm$10.0 & 89.7$\pm$5.0 & 85.3$\pm$1.0 & 40.8$\pm$20.4 & 89.2$\pm$6.7 & 79.9$\pm$0.2 & 95.6$\pm$2.1 & \textbf{97.0}$\pm$1.4\\
			Walker2d-Full-Replay   & 94.2$\pm$1.9 & 97.5$\pm$4.6 & 107.4$\pm$0.6 & 84.8$\pm$13.1 & 88.3$\pm$4.9 & 95.4$\pm$0.7 & 99.9$\pm$3.6 & \textbf{107.8}$\pm$1.2\\
			\midrule
			\textbf{Average} & 73.7 & 81.5 & 85.4 & 65.1 & 68.0 & 88.2 & 81.3 & \textbf{91.2}\\
			\bottomrule
		\end{tabular}
	}
\end{table*}

\paragraph{Dynamics Model Learning}
While the transition dynamics are theoretically treated as a random variable following a continuous distribution, maintaining such a distribution is computationally intractable. To bridge this gap, we approximate the posterior using a model ensemble, where each member is trained via MLE (see Appendix for more details). Under a uniform prior, Equation~\eqref{reweight_metric} and Equation~\eqref{model:eq1} are applied sequentially to each ensemble member, yielding a computationally efficient yet principled posterior approximation.

\begin{figure*}[t!]
	\centering
	% 1
	\begin{subfigure}[b]{0.245\textwidth}
		\includegraphics[width=\linewidth]{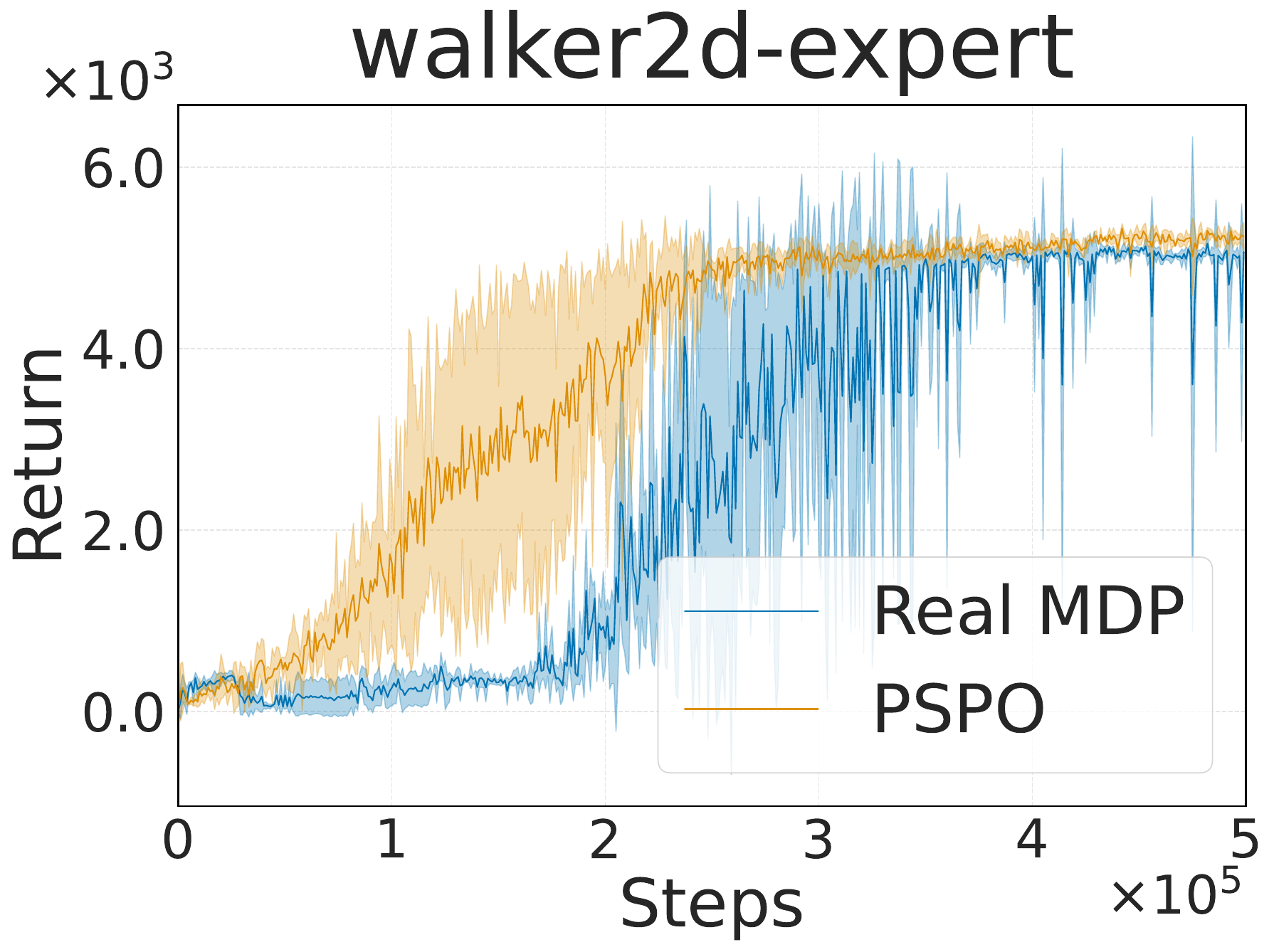}
	\end{subfigure}
	\begin{subfigure}[b]{0.245\textwidth}
		\includegraphics[width=\linewidth]{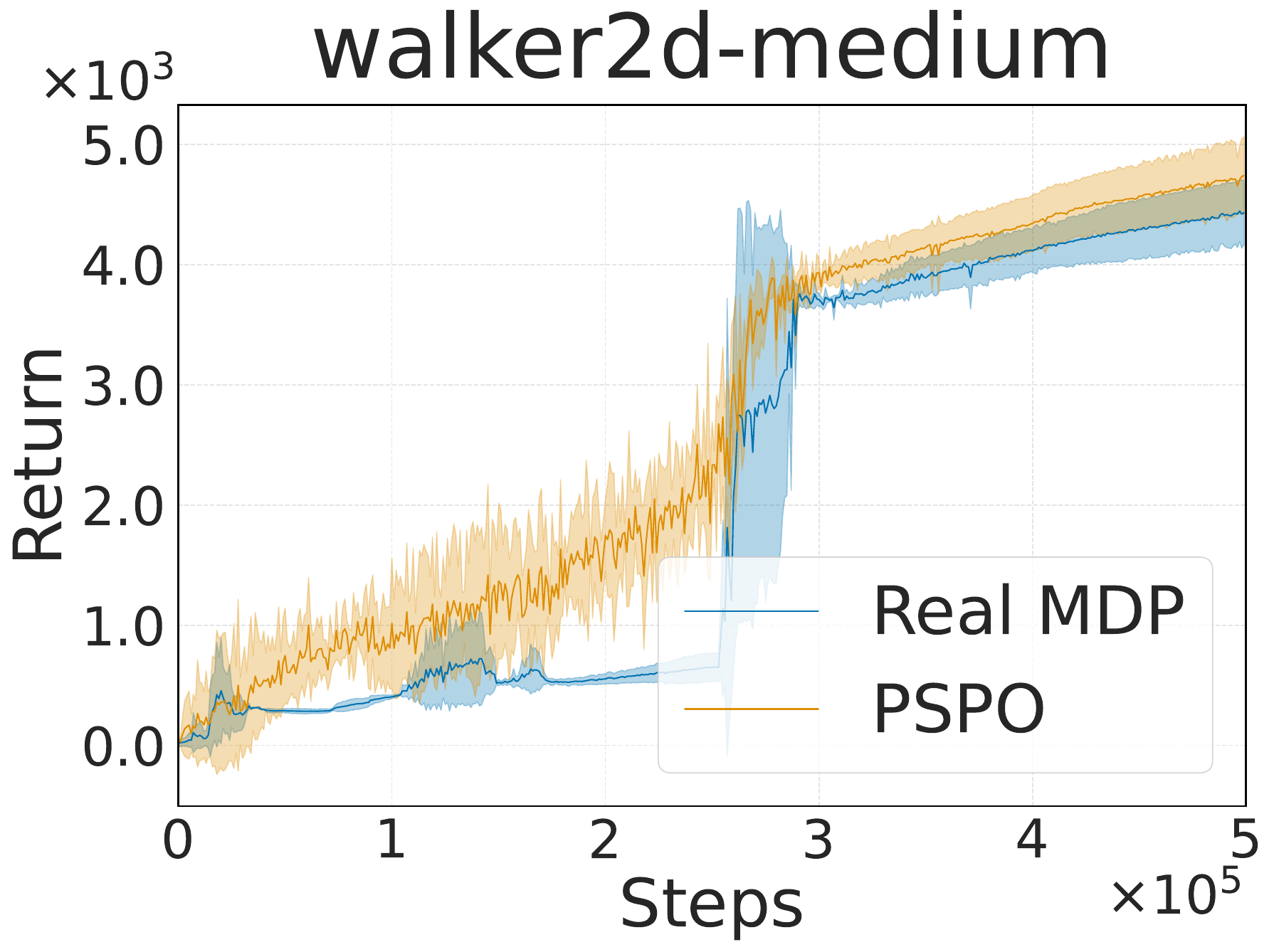} 
	\end{subfigure}
	\centering
	% 2
	\begin{subfigure}[b]{0.245\textwidth}
		\includegraphics[width=\linewidth]{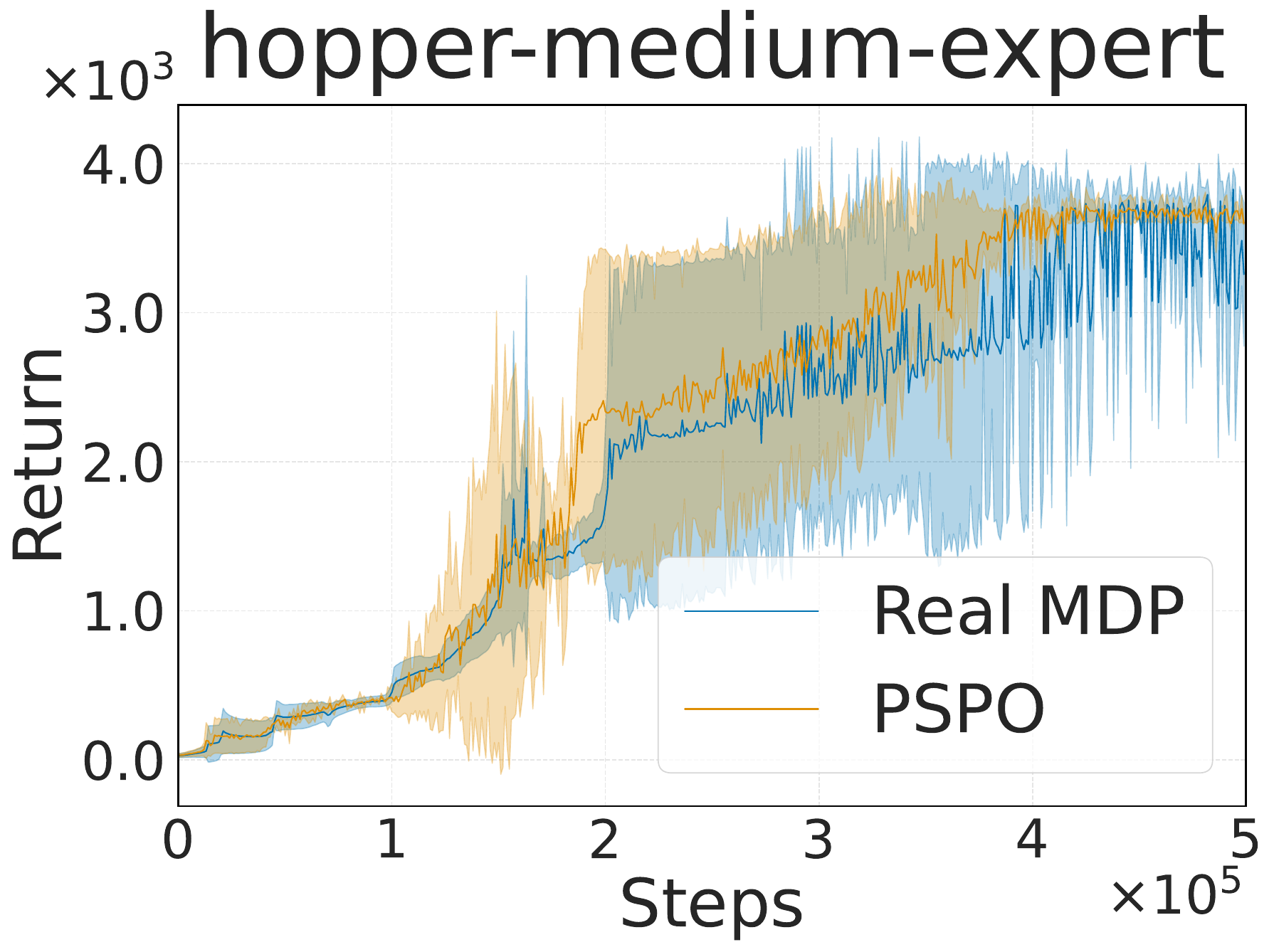} 
	\end{subfigure}
	\begin{subfigure}[b]{0.245\textwidth}
		\includegraphics[width=\linewidth]{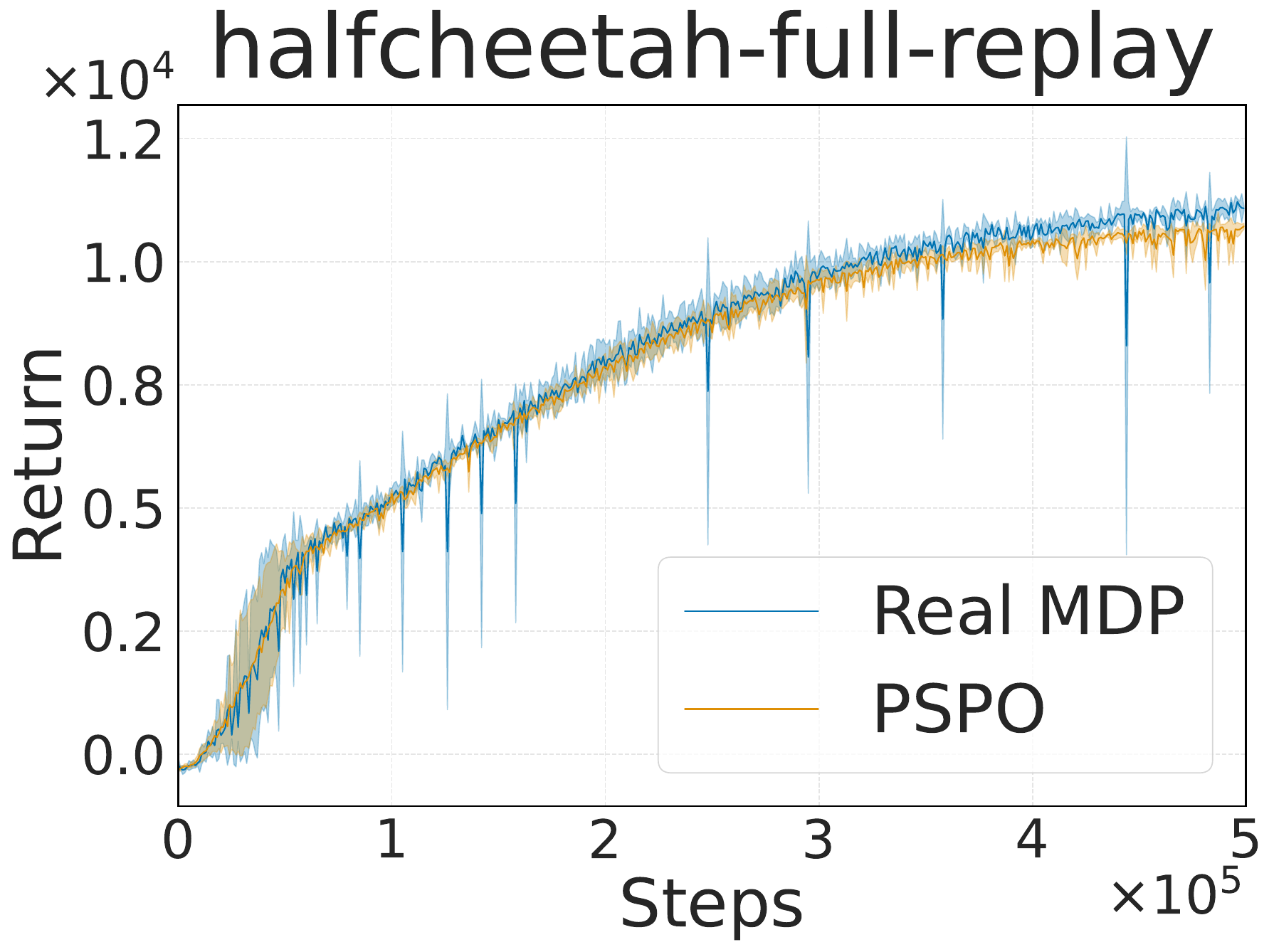} 
	\end{subfigure}
	
	\caption{Learning and evaluation curves. We observe that the true performance improves near-monotonically, even though the returns estimated by the model exhibit no pessimism.}
	\label{exp_fig1:pessimism}
\end{figure*}
\paragraph{Policy Iteration}
We parameterize the Q-function and the policy using neural networks $Q_{\theta}$ and $\pi_{\phi}$, where $\theta$ and $\phi$ denote their respective parameters. To enhance training stability, we employ target networks with soft updates \cite{mnih2015human}. The regularized objective uses the reference policy $\mu$, described here as the behavior policy of the dataset. The theoretical results hold this reference fixed. The target policy $\pi_{\phi'}$ in Algorithm~\ref{algo} is a separate moving-average quantity. We augment the offline dataset $\mathcal{D}$ with the model-generated dataset $\widehat{\mathcal{D}}$ to effectively generalize the policy beyond the offline dataset support \cite{lurevisiting}. The critic uses targets based on Equation~\eqref{opt:eq2}. A Lagrangian penalty provides a practical relaxation of the policy constraint in Equation~\eqref{opt:eq3}. 

\section{Experiments}
Our experimental evaluation focuses on the following four research questions:

\textit{Q1: Performance Evaluation.} What is the empirical performance of PSPO compared to SoTA offline RL algorithms?

\textit{Q2: Empirical Training Behavior.} How do estimated returns, evaluation returns, and uncertainty behave during neural training?

\textit{Q3: Ablation Study.} How does each component of PSPO contribute to overall performance?

\textit{Q4: Discussion.} Evolution of the PSPO posterior during training and computational cost analysis.

To comprehensively address these questions, we conduct extensive evaluations on the D4RL benchmark \cite{fu2020d4rl} and the challenging offline optimal liquidation environment \cite{bao2019multi,rigter2023one}. Detailed hyperparameter settings and environmental configurations are specified in Appendix.

\begin{figure*}[t!]
	\centering
	% 1
	\begin{subfigure}[b]{0.245\textwidth}
		\includegraphics[width=\linewidth]{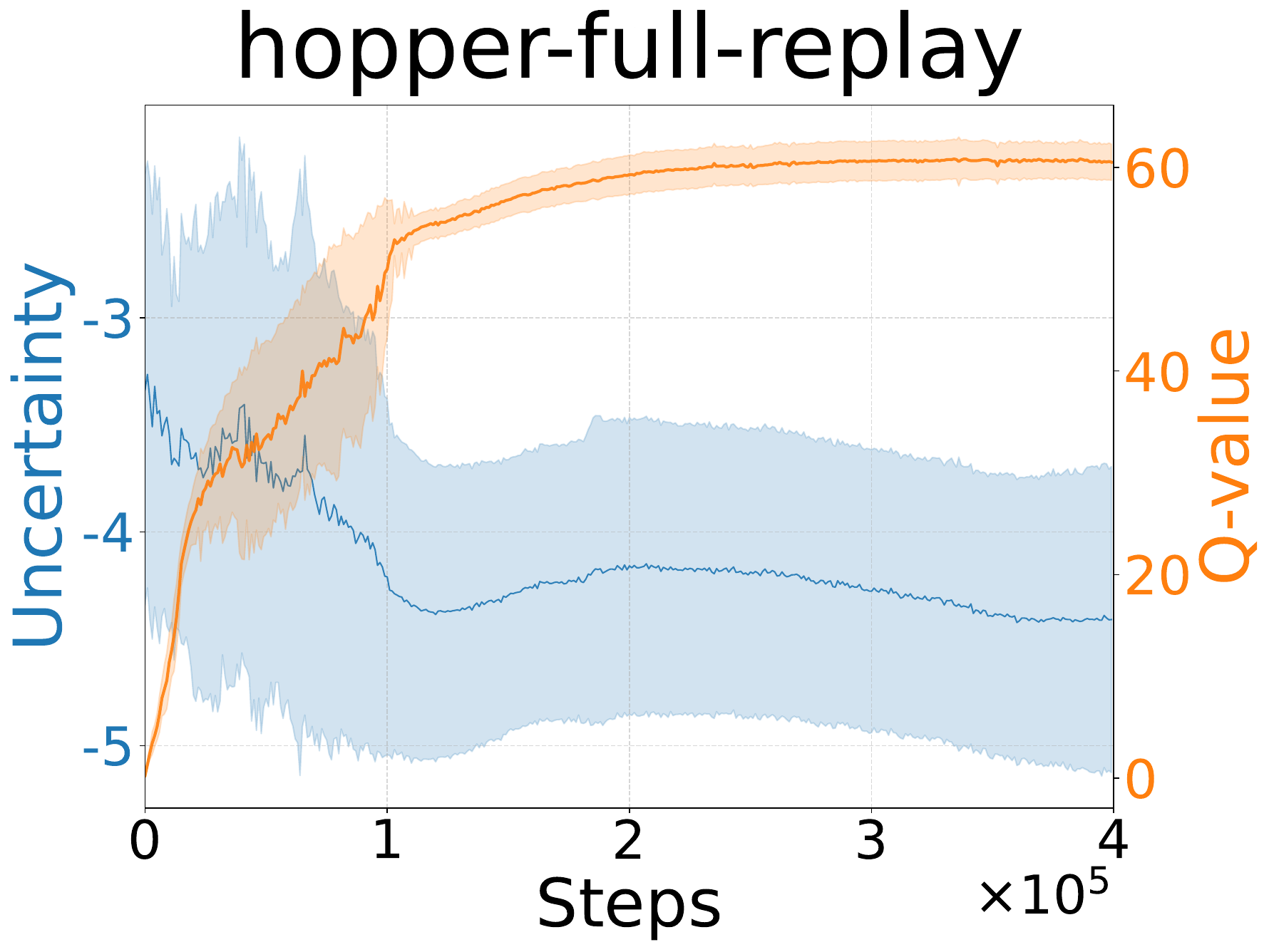}
	\end{subfigure}
	\begin{subfigure}[b]{0.245\textwidth}
		\includegraphics[width=\linewidth]{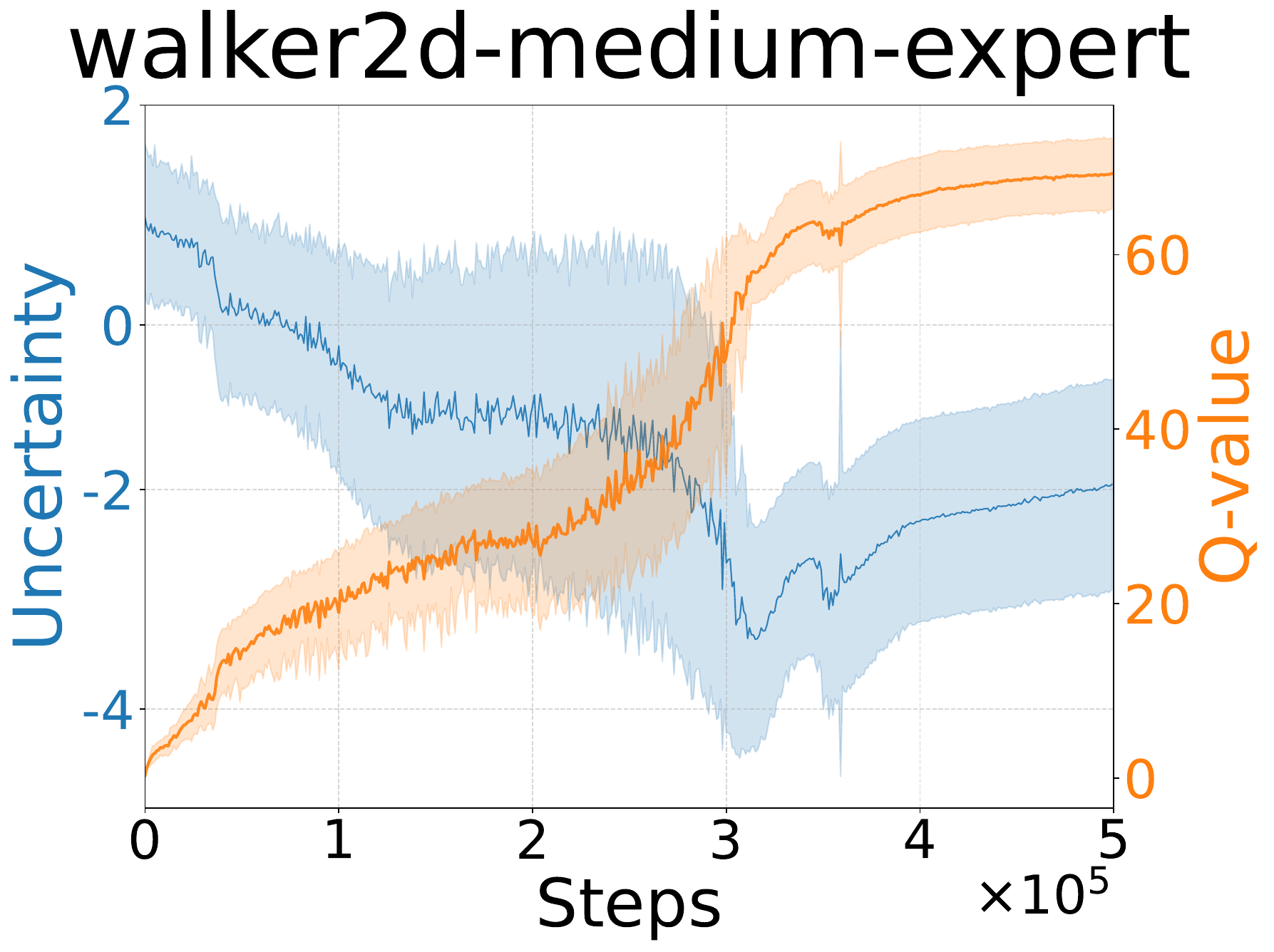} 
	\end{subfigure}
	\centering
	% 2
	\begin{subfigure}[b]{0.245\textwidth}
		\includegraphics[width=\linewidth]{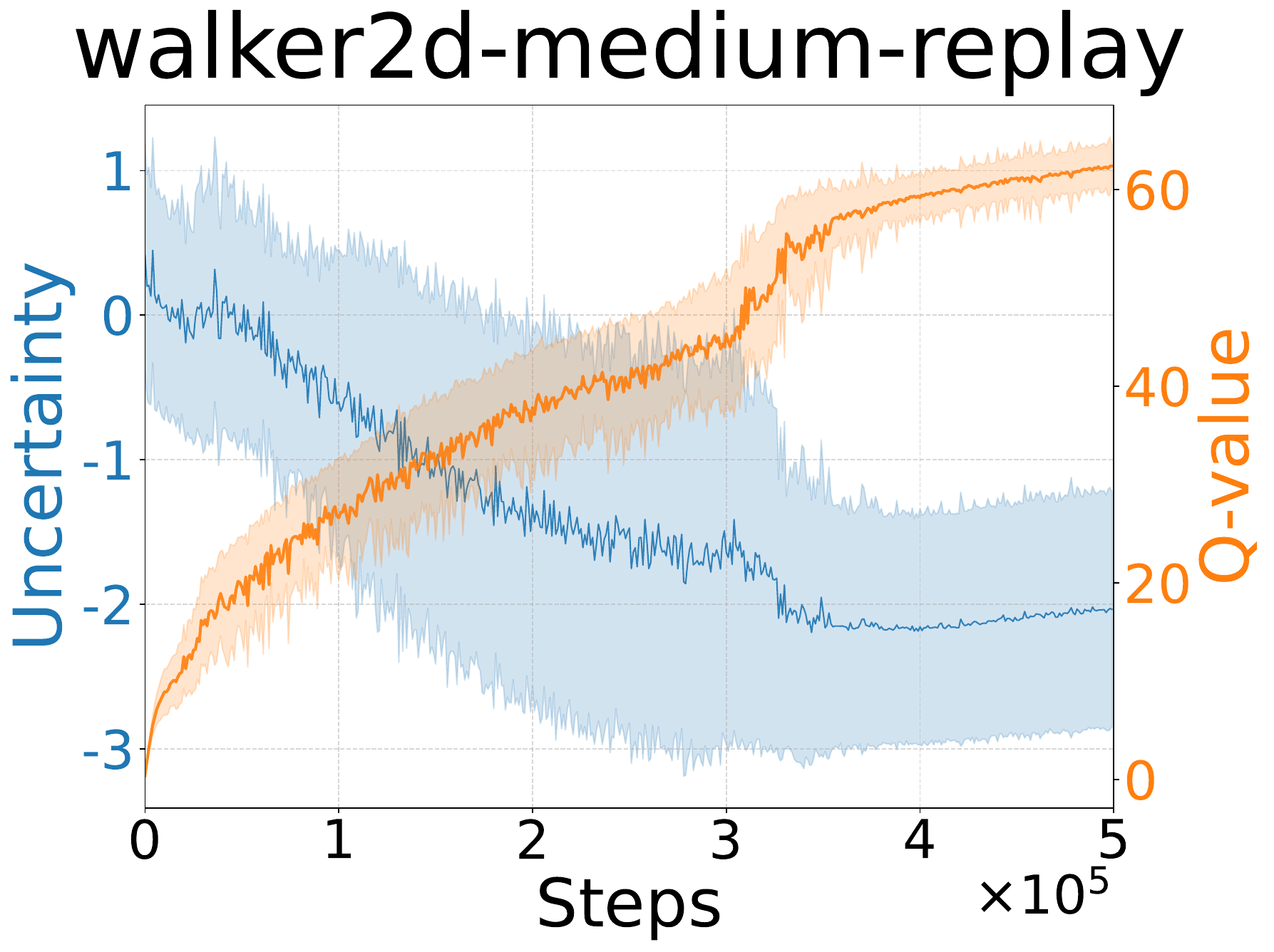} 
	\end{subfigure}
	\begin{subfigure}[b]{0.245\textwidth}
		\includegraphics[width=\linewidth]{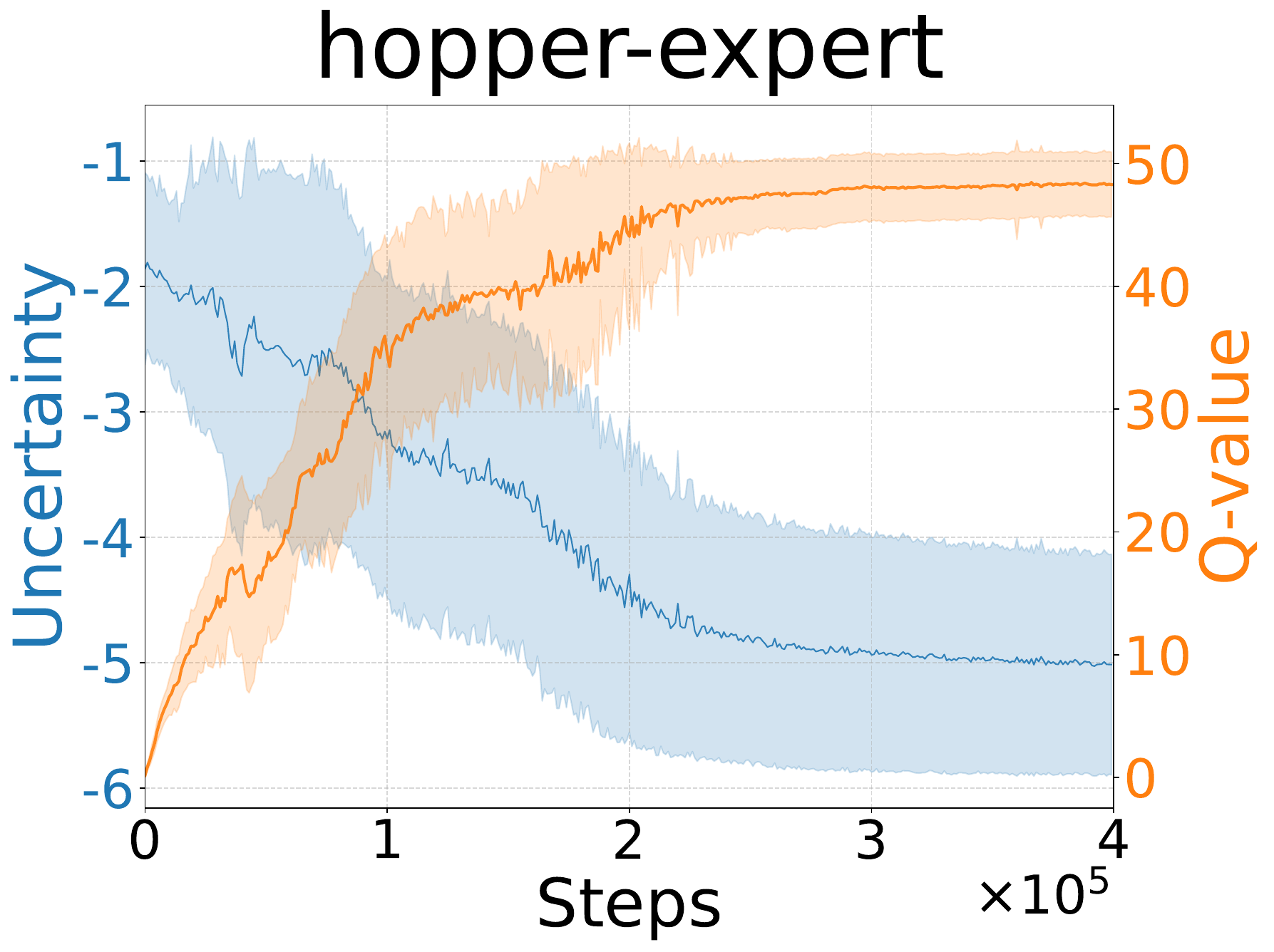} 
	\end{subfigure}
	
	\caption{Correlation between uncertainty and the TD target. The target is defined as $Q(s,a)=r(s,a)+\gamma\mathbb{E}_{s'\sim T, a'\sim\pi}[Q(s',a')]$. Here, uncertainty for a pair $(s,a)$ is quantified as $\log\big(\mathrm{std}(\mathbb{E}_{s'\sim T}[s'])\big)$ with $T \sim P(T|E)$.}
	\label{exp_fig2:uncertainty}
\end{figure*}

\subsection{Performance Evaluation (Q1)}
To answer Q1, we benchmark PSPO against leading model-free (e.g., CQL \cite{kumar2020conservative}, EPQ \cite{NEURIPS2024_fdb11be1}, DMG  \cite{maodoubly}) and model-based approaches (e.g., MOReL \cite{kidambi2020morel}, RAMBO \cite{rigter2022rambo}, PMDB  \cite{guo2022model}, ADM \cite{lin2025anystep}). As shown in Table~\ref{tab:d4rl_performance}, PSPO achieves superior performance on 13 datasets and remains competitive on the remaining 5. Results are sourced from original papers or reproduced using official configurations where necessary. Notably, PSPO demonstrates substantial gains on datasets containing a limited number of high-quality trajectories (e.g., ``medium'', ``medium-expert'', and ``full-replay''). To further assess the results, a paired t-test against the strongest baseline (PMDB) yields $t=2.019$ and $p \approx 0.029$, confirming that the observed improvements are statistically significant at the $5\%$ level.

\begin{table*}[t!]
	\centering
	\caption{Ablation study results. We report the percentage drop in scores ($\downarrow$) and the fold change in standard deviation ($\uparrow/\downarrow $). The notation `x' indicates the fold change (e.g., $\uparrow$9x represents a 9-fold increase). We observe that ablating either component leads to substantial  performance degradation and a notable increase in standard deviation.}
	\label{tab:ablation_study}
	\setlength{\tabcolsep}{13.5pt} 
	\adjustbox{width=0.99\textwidth}{
		\begin{tabular}{lcccc}
			\toprule
			\textbf{Dataset Type} & \textbf{Ablation Setting} & \textbf{HalfCheetah} & \textbf{Hopper} & \textbf{Walker2d} \\
			\midrule
			
			% --- Random ---
			\multirow{2}{*}{Random} 
			& Average Utilization & $\downarrow$19.1\% ($\uparrow$9x)   & $\downarrow$68.1\% ($\downarrow$0.5x) & $\downarrow$14.5\% ($\uparrow$35x) \\
			& Without Regularization & $\downarrow$13.3\% ($\uparrow$1.7x) & $\downarrow$68.1\% ($\uparrow$1.1x)   & $\downarrow$13.9\% ($\uparrow$53x) \\
			\cmidrule(lr){1-5} 
			
			% --- Medium ---
			\multirow{2}{*}{Medium} 
			& Average Utilization & $\downarrow$5.9\% ($\uparrow$2.2x)  & $\downarrow$68.1\% ($\downarrow$0.9x) & $\downarrow$3.9\% ($\downarrow$0.96x) \\
			& Without Regularization & $\downarrow$5.8\% ($\uparrow$1.6x)  & $\downarrow$13.2\% ($\uparrow$9.5x)   & $\downarrow$9.1\% ($\uparrow$1.6x) \\
			\cmidrule(lr){1-5}
			
			% --- Expert ---
			\multirow{2}{*}{Expert} 
			& Average Utilization & $\downarrow$11.2\% ($\uparrow$16x)  & $\downarrow$53.4\% ($\uparrow$4.4x)   & $\downarrow$67.1\% ($\uparrow$13.5x) \\
			& Without Regularization & $\downarrow$23.9\% ($\uparrow$4.4x)  & $\downarrow$20.0\% ($\uparrow$14.9x)  & $\downarrow$78.6\% ($\uparrow$2x) \\
			\cmidrule(lr){1-5}
			
			% --- Medium-Expert ---
			\multirow{2}{*}{Medium-Expert} 
			& Average Utilization & $\downarrow$3.9\% ($\uparrow$3.1x)   & $\downarrow$9.9\% ($\uparrow$5.4x)    & $\downarrow$16.3\% ($\uparrow$12x) \\
			& Without Regularization & $\downarrow$17.3\% ($\uparrow$1.5x)  & $\downarrow$12.8\% ($\uparrow$8x)     & $\downarrow$8.3\% ($\uparrow$3.6x) \\
			\cmidrule(lr){1-5}
			
			% --- Medium-Replay ---
			\multirow{2}{*}{Medium-Replay} 
			& Average Utilization & $\downarrow$9.4\% ($\uparrow$4x)     & $\downarrow$67.9\% ($\uparrow$1.3x)   & $\downarrow$1.7\% ($\uparrow$2.2x) \\
			& Without Regularization & $\downarrow$9.1\% ($\uparrow$2.1x)    & $\downarrow$4.7\% ($\uparrow$6.8x)    & $\downarrow$8.1\% ($\uparrow$1.1x) \\
			\cmidrule(lr){1-5}
			
			% --- Full-Replay ---
			\multirow{2}{*}{Full-Replay} 
			& Average Utilization & $\downarrow$5.6\% ($\uparrow$1.9x)   & $\downarrow$11.1\% ($\uparrow$6x)     & $\downarrow$4.8\% ($\uparrow$1.9x) \\
			& Without Regularization & $\downarrow$7.4\% ($\uparrow$1.6x)   & $\downarrow$10.3\% ($\uparrow$10.4x)  & $\downarrow$9.3\% ($\uparrow$1.5x) \\
			
			\bottomrule
		\end{tabular}
	}
\end{table*}
\subsection{Empirical Training Behavior (Q2)}
\paragraph{Pessimism-free}
Figure~\ref{exp_fig1:pessimism} demonstrates that PSPO effectively avoids excessive pessimism, as evidenced by model-predicted returns consistently exceeding actual performance while maintaining monotonic improvement. These results empirically validate Theorems~\ref{thm2} and \ref{thm3}, confirming that our divergence constraint serves to stabilize optimization and prevent model exploitation rather than to restrict exploration in dynamics-consistent OOD regions.

\paragraph{Robustness}
Figure~\ref{exp_fig2:uncertainty} reports a negative association between the TD target and the uncertainty metric (Spearman's rank correlation, denoted $\tau=-0.2995$). This confirms that model disagreement induces an adaptive posterior reweighting that mitigates overestimation. Rather than using indiscriminate pessimism, PSPO selectively penalizes state-action pairs that deviate from the data distribution to guide the policy back to dynamics-consistent regions, successfully balancing generalization with OOD robustness.

\begin{figure}[t!]
	\centering
	\includegraphics[width=0.49\textwidth]{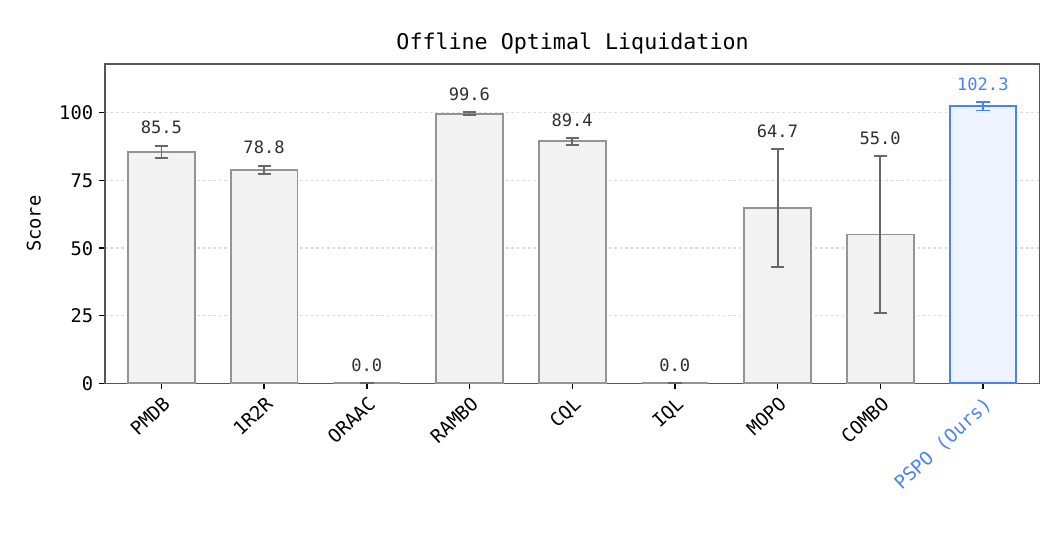}
	\caption{Performance comparison on offline optimal liquidation task.}
	\label{fig:liquidation}
\end{figure}
\paragraph{Generalization}
To assess performance under high uncertainty, we employ a currency liquidation environment in which an agent must liquidate assets over a fixed horizon $T$ amid stochastic rate fluctuations. In this environment, the behavioral policy is heavily biased toward ``hold" actions (80\%). The underlying market dynamics follow a coherent Ornstein-Uhlenbeck process (details in Appendix). We benchmark PSPO against additional model-free (ORAAC \cite{urp2021riskaverse}, IQL \cite{kostrikov2021offline}) and model-based (MOPO \cite{yu2020mopo}, COMBO \cite{yu2021combo}) methods. As summarized in Figure~\ref{fig:liquidation}, while existing methods struggle to converge in this highly stochastic setting, PSPO consistently achieves superior stability and higher cumulative returns. 
\subsection{Ablation Study (Q3)}
To isolate the impact of PSPO's distinct components, we conduct an ablation study focusing on the posterior sampling mechanism (Equation~\eqref{reweight_metric}-Equation~\eqref{model:eq1}) and the reference policy regularization (Equation~\eqref{opt:eq1}). For the posterior sampling component, we disable the reweighting process based on the evidence term $\mathcal{F}(T)$. Instead, we enforce a uniform posterior distribution over the dynamics models, designating this variant as ``Average Utilization."  Regarding the regularization component, we exclude the reference policy $\mu$ entirely, referring to this variant as ``Without Regularization."  As shown in Table~\ref{tab:ablation_study}, ablating either component significantly degrades performance and increases training instability, evidenced by a larger standard deviation.
\subsection{Discussion}

\paragraph{Posterior Concentration Across Training}

We quantify the concentration of the model posterior using its entropy,
\begin{equation}
H(P(T\mid E))
=-\sum_{i=1}^{N}P(T_i\mid E)\log P(T_i\mid E),
\end{equation}
and report $\exp(H)$ as the effective number of models. With the $N=10$ models used by PSPO, a uniform posterior has $H_{\max}=\log 10\approx 2.30$ and $\exp(H)=10$. Table~\ref{tab:posterior_entropy_beta} reports these quantities for the Medium, Medium-Expert, and Medium-Replay datasets. The early, middle, and late phases correspond to fewer than $0.3$M, $0.3$M--$1.5$M, and $1.5$M--$2$M updates, respectively.

\begin{table}[t]
\centering
\caption{Posterior entropy $H$ and effective model count $\exp(H)$
(in parentheses) across training. Early, middle, and late denote
$<0.3$M, $0.3$M--$1.5$M, and $1.5$M--$2$M updates, respectively.
Each task is an environment--dataset pair.}
\label{tab:posterior_entropy_beta}
\footnotesize
\setlength{\tabcolsep}{2pt}
\renewcommand{\arraystretch}{0.95}
\begin{tabular*}{\columnwidth}{@{\extracolsep{\fill}}llccc@{}}
\toprule
Dataset & $\beta$ & Early & Middle & Late \\
\midrule
\multicolumn{5}{@{}l}{\texttt{hopper}} \\
\texttt{medium-expert} & 1 & 1.93 (6.9) & 1.98 (7.2) & 2.01 (7.5) \\
 & 3 & 1.81 (6.1) & 1.83 (6.2) & 1.85 (6.4) \\
 & 5 & 1.78 (5.9) & 1.80 (6.0) & 1.81 (6.1) \\
\addlinespace[1pt]
\texttt{medium} & 1 & 1.64 (5.2) & 1.59 (4.9) & 1.32 (3.7) \\
 & 3 & 1.22 (3.4) & 1.18 (3.3) & 1.05 (2.9) \\
 & 5 & 1.09 (3.0) & 1.07 (2.9) & 0.99 (2.7) \\
\addlinespace[1pt]
\texttt{medium-replay} & 1 & 1.31 (3.7) & 1.37 (3.9) & 1.20 (3.3) \\
 & 3 & 1.03 (2.8) & 1.06 (2.9) & 0.99 (2.7) \\
 & 5 & 0.97 (2.6) & 0.99 (2.7) & 0.95 (2.6) \\
\midrule
\multicolumn{5}{@{}l}{\texttt{halfcheetah}} \\
\texttt{medium-expert} & 1 & 0.28 (1.3) & 0.54 (1.7) & 0.59 (1.8) \\
 & 3 & 0.09 (1.1) & 0.18 (1.2) & 0.20 (1.2) \\
 & 5 & 0.06 (1.1) & 0.11 (1.1) & 0.12 (1.1) \\
\addlinespace[1pt]
\texttt{medium} & 1 & 0.29 (1.3) & 0.39 (1.5) & 0.70 (2.0) \\
 & 3 & 0.10 (1.1) & 0.13 (1.1) & 0.24 (1.3) \\
 & 5 & 0.06 (1.1) & 0.08 (1.1) & 0.15 (1.2) \\
\addlinespace[1pt]
\texttt{medium-replay} & 1 & 0.51 (1.7) & 0.30 (1.3) & 0.54 (1.7) \\
 & 3 & 0.17 (1.2) & 0.10 (1.1) & 0.18 (1.2) \\
 & 5 & 0.10 (1.1) & 0.06 (1.1) & 0.11 (1.1) \\
\midrule
\multicolumn{5}{@{}l}{\texttt{walker2d}} \\
\texttt{medium-expert} & 1 & 0.35 (1.4) & 0.49 (1.6) & 0.91 (2.5) \\
 & 3 & 0.21 (1.2) & 0.26 (1.3) & 0.42 (1.5) \\
 & 5 & 0.18 (1.2) & 0.21 (1.2) & 0.31 (1.4) \\
\addlinespace[1pt]
\texttt{medium} & 1 & 0.37 (1.4) & 0.73 (2.1) & 1.22 (3.4) \\
 & 3 & 0.23 (1.3) & 0.36 (1.4) & 0.60 (1.8) \\
 & 5 & 0.21 (1.2) & 0.28 (1.3) & 0.43 (1.5) \\
\addlinespace[1pt]
\texttt{medium-replay} & 1 & 0.61 (1.8) & 0.55 (1.7) & 0.78 (2.2) \\
 & 3 & 0.42 (1.5) & 0.40 (1.5) & 0.48 (1.6) \\
 & 5 & 0.38 (1.5) & 0.37 (1.4) & 0.41 (1.5) \\
\bottomrule
\end{tabular*}
\end{table}

Across all nine tasks and three training phases, entropy decreases as $\beta$ increases from $1$ to $3$ and $5$. Empirical results show that $\beta=1$ permits concentrated selection on certain tasks while retaining a broader posterior on others. The posterior odds satisfy
\begin{equation}
\frac{P(T_i\mid E)}{P(T_j\mid E)}
=
\frac{P_0(T_i)}{P_0(T_j)}
\exp\!\left[
-\beta\bigl(\mathcal{F}(T_i)-\mathcal{F}(T_j)\bigr)
\right],
\end{equation}
which shows that $\beta$ scales the effect of differences in model residuals on posterior concentration. In the separate return sensitivity study, increasing $\beta$ from $1$ to $3$ or $5$ slightly reduces performance on Expert and Medium-Expert datasets, while changes on Random and Medium datasets are small and mixed. Together, these results motivate using $\beta=1$ for all tasks without task-specific tuning.

\paragraph{Computational overhead.}
Training dominates PSPO's per-iteration computation. On an NVIDIA RTX 3090 Ti GPU, PSPO requires $112.078\pm0.619$ ms per training iteration, the largest training time among the compared methods (Table~\ref{tab:computational_overhead}). Its policy optimization requires forward passes through every dynamics model in the ensemble, accounting for this training overhead. At inference, PSPO selects actions without relying on dynamics-model predictions, resulting in low computational cost.
\begin{table}[t]
\centering
\caption{Computational overhead.}
\label{tab:computational_overhead}
\footnotesize
\setlength{\tabcolsep}{3pt}
\begin{tabular*}{\columnwidth}{@{\extracolsep{\fill}}lcc@{}}
\toprule
Method & Train (ms/iteration) & Inference (ms/iteration) \\
\midrule
PSPO (\emph{Model-based})  & $112.078\pm0.619$ & $3.038\pm0.042$ \\
PMDB (\emph{Model-based})  & $80.853\pm0.408$  & $2.921\pm0.031$ \\
ADM (\emph{Model-based})   & $26.365\pm0.532$  & $9.162\pm0.215$ \\
DMG (\emph{Model-free})    & $10.618\pm0.173$ & $1.080\pm0.030$ \\
TD3+BC (\emph{Model-free}) & $7.265\pm0.140$  & $1.124\pm0.003$ \\
\bottomrule
\end{tabular*}
\end{table}

\section{Conclusion}
 We present PSPO, a pessimism-free model-based offline RL framework that balances generalization with robustness against model exploitation. It combines posterior-sampled rollouts with regularized policy updates to control behavioral deviation without an explicit uncertainty penalty. Experimental results demonstrate that PSPO achieves SoTA performance across diverse benchmarks. A promising direction involves utilizing generative modeling to learn data-driven informative priors, thereby reducing the reliance on manual specifications or standard uninformative priors. We anticipate that this integration will bolster the generalization of offline policies and facilitate more sample-efficient learning in high-dimensional environments.

\clearpage
\newpage

\bibliographystyle{plainnat}
\setcitestyle{numbers}
\bibliography{pspo_arxiv_ref}

\clearpage
\newpage
\beginappendix

\section{Preliminaries: Bayesian RL and Model Sampling}
\label{appendix:knowledge_of_Bayesian_RL}
% Restatements retain the main-text result numbers and use local labels.
\setcounter{equation}{0}
\renewcommand{\theequation}{A\arabic{equation}}
\renewcommand{\theHequation}{appendix.\arabic{equation}}
\setcounter{table}{0}
\renewcommand{\thetable}{A\arabic{table}}
\renewcommand{\theHtable}{appendix.\arabic{table}}
\providecommand{\theHtheorem}{}
\renewcommand{\theHtheorem}{appendix.\arabic{theorem}}
\providecommand{\theHlemma}{}
\renewcommand{\theHlemma}{appendix.\arabic{lemma}}
\providecommand{\theHproposition}{}
\renewcommand{\theHproposition}{appendix.\arabic{proposition}}
\providecommand{\theHcorollary}{}
\renewcommand{\theHcorollary}{appendix.\arabic{corollary}}
\setcounter{lemma}{1}

We consider a discounted MDP with state space $\mathcal S$, action space $\mathcal A$, initial distribution $\rho_0$, shared reward $r$, and discount $\gamma\in(0,1)$. Each transition model is a probability kernel $T:\mathcal S\times\mathcal A\to\Delta(\mathcal S)$, and $P$ is a distribution over candidate models. Bayesian updates from observed transitions, model sampling during rollouts, and critic-dependent posterior reconstruction serve distinct roles.

\subsection{Definitions}
\label{app:setting}
For all convergence and policy-improvement results below, $\mathcal S$, $\mathcal A$, and the model set $\mathcal T=\operatorname{supp}(P_0)$ are finite. The models and prior are fixed, $P_0(T)>0$ on $\mathcal T$, and $|r(s,a)|\le R_{\max}$. Let $Q_{\max}\ge R_{\max}/(1-\gamma)$ and $\mathcal Q=\{Q:\|Q\|_\infty\le Q_{\max}\}$, where $\|Q\|_\infty=\max_{s,a}|Q(s,a)|$. The fixed reference policy satisfies $\mu(a|s)>0$, and $\alpha,\beta>0$. Define $\Pi=\prod_s\Delta(\mathcal A)$ and $d(\pi,\pi')=\max_s\sum_a|\pi(a|s)-\pi'(a|s)|$.

The optimization objective for PSPO is to maximize the expected return averaged over the posterior distribution of the model dynamics:
\begin{equation}
	\label{app:pre:eq4}
	\mathcal{J}(\pi)=\mathop {\mathbb E}_{\substack{\rho_0, \pi\\
			{T_0\sim P(T|E)}}} \left[\mathop {\mathbb E}\limits_{\substack{T_0,\pi \\ T_1\sim P(T|E)}}\left[\cdots \mathop {\mathbb E}_{T_{\infty}, \pi}\big[\sum_{t=0}^\infty\gamma^tr(s_t,a_t)\big]\right]\right].
\end{equation} 

Write $\mathcal{J}(\pi,T)$ for this return with a single model $T$, and $Q_P^\pi,V_P^\pi$ for the ordinary action and state values under $\bar T_P$. For an arbitrary critic, $V_Q^\pi(s)=\sum_a\pi(a|s)Q(s,a)$, and
\begin{equation}
\label{app:model:operator}
(\mathcal B_P^{\pi}Q)(s,a)=r(s,a)+\gamma\mathbb E_{s'\sim T(\cdot|s,a)}V_Q^\pi(s').
\end{equation}
A subscript $T$ instead of $P$ denotes the operator for that single model. Given a nonempty evidence batch $E=\{(s_j,a_j)\}_{j=1}^b$, the consistency score is
\begin{equation}
\label{app:reweight_metric}
\mathcal F_{Q,\pi}(T;E)=\frac1b\sum_{j=1}^b[Q(s_j,a_j)-(\mathcal B_T^\pi Q)(s_j,a_j)]^2.
\end{equation}
The posterior is reconstructed from the same fixed prior:
\begin{equation}
\label{app:model:eq1}
\begin{aligned}
P_{Q,\pi}&=\arg\min_{P\in\Delta(\mathcal T)}\{D_{\mathrm{KL}}(P\Vert P_0)+\beta\mathbb E_{T\sim P}\mathcal F_{Q,\pi}(T;E)\},\\
P_{Q,\pi}(T)&=\frac{P_0(T)e^{-\beta\mathcal F_{Q,\pi}(T;E)}}{\sum_{T'}P_0(T')e^{-\beta\mathcal F_{Q,\pi}(T';E)}}.
\end{aligned}
\end{equation}
At update $k$, $\mathcal F_k(T;E_k)=\mathcal F_{Q_k,\pi_k}(T;E_k)$, $P_k$ denotes the resulting posterior, and $Z_k$ is its normalizing constant. The evidence batch is held fixed throughout the deterministic alternating convergence analysis.

The regularized return and soft Bellman operator are
\begin{equation}
\label{app:opt:eq1}
\begin{aligned}
\widehat{\mathcal{J}}_P^\alpha(\pi)=\mathop {\mathbb E}_{\substack{\rho_0, \pi\\
		{T_0\sim P(T|E)}}} \bigg[\mathop {\mathbb E}\limits_{\substack{T_0,\pi \\ T_1\sim P(T|E)}}\big[\cdots \mathop {\mathbb E}_{T_{\infty}, \pi}\big[\sum_{t=0}^\infty\gamma^t\big(r(s_t,a_t)-\alpha D_\mathrm{KL}(\pi(\cdot|s_t)||\mu(\cdot|s_t))\big)\big]\big]\bigg],
\end{aligned}
\end{equation}
\begin{equation}
\label{app:opt:eq2}
(\mathcal B_P^*Q)(s,a)=r(s,a)+\gamma\sum_{s'}\bar T_P(s'|s,a)\alpha\log\sum_{a'}\mu(a'|s')e^{Q(s',a')/\alpha}.
\end{equation}
For any $Q$, let $\pi_Q(a|s)=\mu(a|s)e^{Q(s,a)/\alpha}/\sum_{a'}\mu(a'|s)e^{Q(s,a')/\alpha}$. For a fixed full-support state distribution $\nu$, let $D_{\mathrm{KL},\nu}(\pi\Vert\pi_k)=\sum_s\nu(s)D_{\mathrm{KL}}(\pi(\cdot|s)\Vert\pi_k(\cdot|s))$. The separate constrained policy update is
\begin{equation}
\label{app:opt:eq3}
\pi_{k+1}\in\arg\max_{\pi:\,D_{\mathrm{KL},\nu}(\pi\Vert\pi_k)\le\epsilon}\widehat{\mathcal{J}}_{P_k}^\alpha(\pi).
\end{equation}
For policy coordinates $\theta$, write $\mathcal{J}_P(\theta)=\mathcal{J}_P(\pi_\theta)$ and $\widehat{\mathcal{J}}_P^\alpha(\theta)=\widehat{\mathcal{J}}_P^\alpha(\pi_\theta)$. Neural policies use parameters $\phi$ and neural critics use parameters $\theta$; the independent tabular coordinates in the local-improvement theorem are specified separately.

The stochastic critic update, applied synchronously to every coordinate, is
\begin{equation}
\label{app:thm:policy_eval}
Q_{t+1}(s,a)=Q_t(s,a)+\eta_t[Y_t(s,a)-Q_t(s,a)].
\end{equation}
Here $\mathscr H_t$ contains the information available before target sampling at iteration $t$. The quantities $Q_t,\pi_t,P_t$ used for sampling are measurable with respect to this information.

\subsection{Bayesian Inference over Transition Models}
For a fixed reward function, model uncertainty can be represented by a prior $P_0(T)$ over transition kernels. Given observed transitions $\mathcal D_t$, ordinary Bayesian inference gives $P(T\mid\mathcal D_t)\propto P_0(T)P(\mathcal D_t\mid T)$. In a finite model space, observing the next state $s'$ after choosing $a$ at $s$ updates a current belief $P$ to
\begin{equation*}
\Psi(P,s,a,s')(T)
=\frac{T(s'\mid s,a)P(T)}{\sum_{\widetilde T\in\mathcal T}\widetilde T(s'\mid s,a)P(\widetilde T)},
\end{equation*}
whenever the denominator is positive. This expression conditions on the full observed history and on an action chosen using that history. Its denominator is the predictive probability $\bar T_P(s'\mid s,a)$. Integrals replace sums for a continuous model space.

\subsection{Belief-State Planning}
A Bayes-adaptive policy can depend on both the physical state $s$ and the current belief $P$. With known rewards, its optimal state value satisfies
\begin{equation*}
V^*(s,P)=\max_a\left\{r(s,a)+\gamma\sum_{s'}\bar T_P(s'\mid s,a)V^*(s',\Psi(P,s,a,s'))\right\}.
\end{equation*}
The updated belief inside $V^*$ accounts for how future observations affect future decisions. Planning with a frozen $P$ instead uses the fixed predictive kernel $\bar T_P(s'|s,a)=\sum_TP(T)T(s'|s,a)$ and the return $\mathcal{J}_P(\pi)$ from Equation~\eqref{app:pre:eq4}. These are different decision problems. The former includes the value of learning about the environment, whereas the latter evaluates a policy under a specified model distribution \cite{ghosh2022offline}.

Thompson sampling draws a model $T\sim P$ and acts according to a policy optimized for that model. Posterior sampling of transitions can also be used to evaluate or train a shared policy. PSPO uses this latter mechanism with weights determined by critic consistency.

\subsection{Deep Ensembles and Rollout Conventions}
A dynamics ensemble $\{T_{\psi_i}\}_{i=1}^N$ represents several learned transition models. Each member predicts a next-state distribution, while variation across members supplies a model-disagreement diagnostic \cite{chua2018deep}. This representation separates stochasticity within a model from differences among models; the accuracy of the disagreement diagnostic depends on the learned ensemble.

The sampling convention determines the trajectory distribution. Independent model draws at each step give the predictive kernel $\bar T_P$ and objective $\mathcal{J}_P$. Drawing one model per episode and retaining it instead gives $\mathcal{J}_P^{\mathrm{episode}}(\pi)=\mathbb E_{T\sim P}[J(\pi,T)]$. Trajectory-sampling methods use ensemble members to propagate simulated particles \cite{chua2018deep}. The fixed-model convention preserves dependence across transitions within each particle, so it cannot generally be replaced by independent per-transition draws.

\subsection{Critic-Dependent Reconstruction in PSPO}
At update $k$, PSPO compares all candidate models using the same $Q_k$, $\pi_k$, and evidence batch $E_k$. Equation~\eqref{app:reweight_metric} defines the consistency score as the average squared Bellman residual. Equation~\eqref{app:model:eq1} then gives $P_k(T)\propto P_0(T)\exp[-\beta\mathcal F_k(T;E_k)]$. The evidence specifies where to evaluate consistency; it does not change the model parameters during this reweighting step.

For any distribution $P$ on the ensemble, $D_{\mathrm{KL}}(P\Vert P_0)+\beta\mathbb E_P\mathcal F_k=D_{\mathrm{KL}}(P\Vert P_k)-\log Z_k$. Thus $P_k$ is the unique minimizer in Equation~\eqref{app:model:eq1}. This generalized posterior balances the expected consistency score against KL divergence from the prior \cite{bissiri2016general}. It differs from $P(T\mid\mathcal D_t)$ above, whose likelihood is based on observed transitions. With a fixed evidence batch, we write $P_{Q,\pi}$ to expose the dependence on both the critic and scoring policy. Theorem~\ref{app:thm:adaptive} analyzes alternating critic--posterior solves and exact regularized policy updates. Corollary~\ref{app:cor:stochastic} handles bounded, conditionally unbiased target sampling within a fixed-policy critic phase.

\section{KL-Regularized Policy Optimization}
\label{appendix:regularized}
For a fixed state, critic $Q$, reference distribution $\mu(\cdot\mid s)>0$, and $\alpha>0$, the statewise problem underlying Equation~\eqref{app:opt:eq2} is
\begin{equation*}
\max_{\pi(\cdot\mid s)}\left\{\sum_a\pi(a\mid s)Q(s,a)-\alpha D_{\mathrm{KL}}(\pi(\cdot\mid s)\Vert\mu(\cdot\mid s))\right\}.
\end{equation*}
The feasible set is the probability simplex. Define the normalized Gibbs policy
\begin{equation*}
\pi_Q(a\mid s)=\frac{\mu(a\mid s)e^{Q(s,a)/\alpha}}{\sum_{a'}\mu(a'\mid s)e^{Q(s,a')/\alpha}}.
\end{equation*}
Substitution into the KL divergence gives
\begin{equation*}
\sum_a\pi(a\mid s)Q(s,a)-\alpha D_{\mathrm{KL}}(\pi(\cdot\mid s)\Vert\mu(\cdot\mid s))
=V_Q^{\mathrm{soft}}(s)-\alpha D_{\mathrm{KL}}(\pi(\cdot\mid s)\Vert\pi_Q(\cdot\mid s)),
\end{equation*}
where $V_Q^{\mathrm{soft}}(s)=\alpha\log\sum_a\mu(a\mid s)e^{Q(s,a)/\alpha}$. Nonnegativity of the KL divergence shows that the unique maximizer is $\pi_Q$ and the maximum is $V_Q^{\mathrm{soft}}$. This identity establishes the statewise optimization used in the proof of Theorem~\ref{app:thm2}.

The ordinary state value used in posterior scoring remains $V_Q^\pi(s)=\sum_a\pi(a\mid s)Q(s,a)$. It averages under a given policy. The soft value additionally optimizes the action distribution subject to the reference-policy penalty. At a fixed point $Q_P^*$ of the soft operator, its Gibbs policy $\pi_P^*=\pi_{Q_P^*}$ maximizes the regularized return $\widehat{\mathcal{J}}_P^\alpha$.

The constrained policy update in Equation~\eqref{app:opt:eq3} maximizes $\widehat{\mathcal{J}}_{P_k}^\alpha$ while restricting $D_{\mathrm{KL},\nu}(\pi\Vert\pi_k)$. The KL term relative to $\mu$ is part of the objective; the KL term relative to $\pi_k$ controls the update size. The gradient-alignment assumption in Theorem~\ref{app:thm3} connects this regularized update to improvement of the original return $\mathcal{J}_{P_k}$.

\section{Notation Summary}
\label{appendix:notation}
The main-text shorthand $\mathcal J(\pi)$ refers to $\mathcal J_P(\pi)$ with $P$ held fixed; the subscript makes the posterior dependence explicit. The hat and superscript in $\widehat{\mathcal{J}}_P^\alpha$ identify the regularized objective. All proofs below preserve this distinction.
\begin{longtable}{p{0.30\textwidth}p{0.64\textwidth}}
\toprule
Symbol & Meaning \\
\midrule
\endhead
$T,\mathcal T,P_0$ & Transition kernel, finite candidate model set, and fixed prior. \\
$P_k=P(\cdot\mid E_k)$ & Generalized model posterior at update $k$, computed from $Q_k,\pi_k,E_k$. \\
$P_{Q,\pi}$ & Posterior reconstructed from critic $Q$ and scoring policy $\pi$ using fixed evidence. \\
$\bar T_P$ & Predictive kernel $\mathbb E_{T\sim P}T$. \\
$J(\pi,T)$ & Return of $\pi$ under a single transition model $T$. \\
$\mathcal{J}_P(\pi)$ & Original return under per-transition model resampling, equivalently under $\bar T_P$. \\
$\mathcal{J}_P^{\mathrm{episode}}(\pi)$ & Return when one model is sampled and retained for the episode. \\
$\widehat{\mathcal{J}}_P^\alpha(\pi)$ & Return under $\bar T_P$ with the reference-policy KL penalty. \\
$Q_T^\pi,V_T^\pi$ & Exact unregularized action and state values under a single model. \\
$Q_P^\pi,V_P^\pi$ & Exact unregularized values under $\bar T_P$. \\
$Q_k,V_k$ & Current shared critic and its policy average $V_k=V_{Q_k}^{\pi_k}$. \\
$V_Q^\pi,V_Q^{\mathrm{soft}}$ & Ordinary policy average and KL-regularized soft state value. \\
$\mathcal B_P^\pi,\mathcal B_P^*$ & Fixed-posterior evaluation and soft optimality operators. \\
$Q_P^*,\pi_P^*$ & Soft fixed point and its Gibbs policy. \\
$\mu,\alpha$ & Fixed reference policy and policy-regularization strength. \\
$E_k,b,\mathcal F_k,\beta$ & Evidence batch, batch size, consistency score, and posterior inverse temperature. \\
$\nu,\epsilon$ & State weights and radius of the policy-update KL constraint. \\
$g_k,F_k$ & Gradient of $\widehat{\mathcal{J}}_{P_k}^\alpha$ and local Hessian of the policy-update KL. \\
$Q_{\max},\mathcal Q$ & Critic magnitude bound and the invariant sup-norm ball. \\
$D,\Delta,\lambda$ & Residual bound uniform over critics and scoring policies, soft-target spread, and fixed-policy stability ratio. \\
$\Pi,d,\pi_Q$ & Tabular policy space, maximum statewise $\ell_1$ distance, and Gibbs policy of $Q$. \\
$\mathcal A_\pi,\mathcal G_n^\pi$ & One or $n$ soft Bellman steps after posterior reconstruction at fixed scoring policy $\pi$. \\
$\bar Q(\pi),\bar Q_k,\bar P_k$ & Converged critic at fixed $\pi$, and critic and posterior produced by cycle $k$. \\
$\kappa,\rho$ & Contraction bounds for a fixed-policy critic phase and the complete policy cycle. \\
$\pi^\star,Q^\star,P^\star$ & Unique self-consistent limit of the alternating procedure. \\
$\mathscr H_t,Y_t,\eta_t$ & Pre-sampling information, sampled target, and stochastic-approximation step size. \\
$\mathcal A,\zeta$ & Conditional-mean update map and, when batches are resampled, their fixed sampling law. \\
\bottomrule
\end{longtable}

\section{Proofs}
\label{appendix:proof}
\subsection{Policy Evaluation and Sampling Variance}
\label{appendix:proposition1}
\setcounter{proposition}{0}
\begin{proposition}[Fixed-Posterior Policy Evaluation]
\label{app:proposition1}
For fixed $P$ and $\pi$, $\mathcal B_P^\pi$ maps $\mathcal Q$ into itself and is a $\gamma$-contraction in the sup norm. Its unique fixed point is $Q_P^\pi$.
\end{proposition}

\begin{proof}[Proof of Proposition~\ref{app:proposition1}]
For fixed $P,\pi$, both expectations use the same measure, so $\|\mathcal B_P^\pi Q-\mathcal B_P^\pi U\|_\infty\le\gamma\|Q-U\|_\infty$. Also $\|\mathcal B_P^\pi Q\|_\infty\le R_{\max}+\gamma Q_{\max}\le Q_{\max}$ for $Q\in\mathcal Q$. The Banach fixed-point theorem gives a unique limit, which is the policy-evaluation value $Q_P^\pi$ under $\bar T_P$.
\end{proof}
\setcounter{proposition}{1}
\begin{proposition}[Conditional Sampling Variance]
\label{app:thm:var}
Let $\mathscr H_t$ denote the information available before sampling at iteration $t$. Suppose $Q_t\in\mathcal Q$, $\pi_t$, and $P_t$ are determined by this history. If $T_t\sim P_t$ and $Y_t(s,a)=(\mathcal B_{T_t}^{\pi_t}Q_t)(s,a)$, then $\mathbb E[Y_t|\mathscr H_t]=\mathcal B_{P_t}^{\pi_t}Q_t$ and $\operatorname{Var}(Y_t(s,a)|\mathscr H_t)\le Q_{\max}^2$, including for critic-dependent $P_t$.
\end{proposition}

\begin{proof}[Proof of Proposition~\ref{app:thm:var}]
Conditioning on $\mathscr H_t$ fixes the critic, policy, and model probabilities before sampling. Averaging $Y_t=\mathcal B_{T_t}^{\pi_t}Q_t$ over $T_t\sim P_t$ gives its stated conditional mean. Since $|Y_t(s,a)|\le R_{\max}+\gamma Q_{\max}\le Q_{\max}$, its conditional variance is at most $\mathbb E[Y_t(s,a)^2|\mathscr H_t]\le Q_{\max}^2$. This allows any history-measurable $P_t$; it requires a bounded critic as well as bounded rewards.
\end{proof}

\subsection{The Fixed-Posterior Regularized Operator}
\label{appendix:thm:operator}
\begin{lemma}
\label{lemma1}
For a fixed probability distribution $\mu(\cdot|s)$ and $\alpha>0$, the map $V_Q^{\mathrm{soft}}(s)=\alpha\log\sum_a\mu(a|s)e^{Q(s,a)/\alpha}$ is 1-Lipschitz in the sup norm and lies in $[-Q_{\max},Q_{\max}]$ whenever $Q\in\mathcal Q$.
\end{lemma}
\begin{proof}
If $d=\|Q-U\|_\infty$, then $Q(s,a)\le U(s,a)+d$, so $V_Q^{\mathrm{soft}}(s)\le V_U^{\mathrm{soft}}(s)+d$. Interchanging $Q$ and $U$ proves the Lipschitz bound. Comparing with the constant functions $-Q_{\max}$ and $Q_{\max}$ proves the range bound.
\end{proof}

\setcounter{theorem}{0}
\begin{theorem}[Fixed-Posterior Regularized Optimality]
\label{app:thm2}
For fixed $P,\mu,\alpha$, $\mathcal B_P^*$ is a $\gamma$-contraction with a unique fixed point $Q_P^*$. Its Gibbs policy $\pi_P^*(a|s)\propto\mu(a|s)e^{Q_P^*(s,a)/\alpha}$ maximizes $\widehat{\mathcal{J}}_P^\alpha$.
\end{theorem}

\begin{proof}[Proof of Theorem~\ref{app:thm2}]
With $P$ fixed, Lemma~\ref{lemma1} gives
\begin{equation*}
\|\mathcal B_P^*Q-\mathcal B_P^*U\|_\infty
\le\gamma\|V_Q^{\mathrm{soft}}-V_U^{\mathrm{soft}}\|_\infty
\le\gamma\|Q-U\|_\infty.
\end{equation*}
Hence the operator has a unique fixed point $Q_P^*$. For any $Q$, define
\begin{equation*}
\pi_Q(a|s)=\frac{\mu(a|s)e^{Q(s,a)/\alpha}}{\sum_{a'}\mu(a'|s)e^{Q(s,a')/\alpha}}.
\end{equation*}
Direct expansion of the KL divergence gives, for any policy $\pi$,
\begin{equation*}
\sum_a\pi(a|s)Q(s,a)-\alpha D_{\mathrm{KL}}(\pi(\cdot|s)\Vert\mu(\cdot|s))
=V_Q^{\mathrm{soft}}(s)-\alpha D_{\mathrm{KL}}(\pi(\cdot|s)\Vert\pi_Q(\cdot|s)).
\end{equation*}
Thus $\pi_Q$ attains the statewise maximum. At $Q=Q_P^*$, let $V^*=V_{Q_P^*}^{\mathrm{soft}}$. The fixed-point equation and the variational identity imply that $V^*$ is fixed by the regularized policy-evaluation operator for $\pi_P^*$. For any other $\pi$, its regularized evaluation operator maps $V^*$ to a value no larger than $V^*$. Monotonicity and contraction of that fixed-policy evaluation operator yield $V_P^{\pi,\mathrm{reg}}\le V^*$ by repeated application, where $V_P^{\pi,\mathrm{reg}}$ denotes the regularized state value of $\pi$. Taking the expectation under $\rho_0$ proves optimality for $\widehat{\mathcal{J}}_P^\alpha$. The posterior and reference policy remain fixed throughout this argument.
\end{proof}

\subsection{Alternating Posterior and Policy Updates}
\label{appendix:adaptive}
Throughout this subsection, the evidence batch, learned models, prior, reference policy, and temperatures are fixed. For $E=\{(s_j,a_j)\}_{j=1}^b$, define $\widehat y_m(Q,\pi;s,a)=(\mathcal B_{T_m}^\pi Q)(s,a)$ and $y_m(Q;s,a)=(\mathcal B_{T_m}^*Q)(s,a)$. The consistency score and posterior weights are $\mathcal F_{Q,\pi}(T_m;E)=b^{-1}\sum_j[Q(s_j,a_j)-\widehat y_m(Q,\pi;s_j,a_j)]^2$ and $p_m(Q,\pi)=P_{Q,\pi}(T_m)\propto P_0(T_m)e^{-\beta\mathcal F_{Q,\pi}(T_m;E)}$. Models are indexed over the prior support, so $P_0(T_m)>0$. The bound $D$ is uniform over $Q\in\mathcal Q$ and $\pi\in\Pi$; $\Delta$ bounds the span of the soft targets across models, uniformly over $Q,s,a$.

\paragraph{Alternating Procedure and Uniform Bounds}
We use the setting of Section~\ref{app:setting}. The evidence batch $E$, model ensemble, prior $P_0$, reference policy $\mu$, and temperatures $\alpha,\beta$ remain fixed across both stages and all cycles. Let $\Pi=\prod_{s\in\mathcal S}\Delta(\mathcal A)$ denote the tabular policy space, equipped with $d(\pi,\pi')=\max_s\sum_a|\pi(a|s)-\pi'(a|s)|$. Each critic phase starts in $\mathcal Q=\{Q:\|Q\|_\infty\le Q_{\max}\}$, where $\|f\|_\infty=\max_{s,a}|f(s,a)|$.

Define the uniform residual and model-disagreement bounds
\begin{equation*}
\begin{aligned}
D&=\sup_{\pi\in\Pi,\,Q\in\mathcal Q}\max_{T\in\mathcal T}\|Q-\mathcal B_T^\pi Q\|_\infty,\\
\Delta&=\sup_{Q\in\mathcal Q}\max_{T,T'\in\mathcal T}\|\mathcal B_T^*Q-\mathcal B_{T'}^*Q\|_\infty.
\end{aligned}
\end{equation*}
Here $D$ bounds the ordinary Bellman residual used in the consistency score $\mathcal F_{Q,\pi}(T;E)$, uniformly over critics and scoring policies. The bound $\Delta$ measures the largest difference between two models' soft targets at the same critic and state-action pair; $\mathcal B_T^*$ uses model $T$ alone.

Set $\lambda=\beta D\Delta(1+\gamma)/(1-\gamma)$, $\kappa=\gamma+(1-\gamma)\lambda$, and $\mathcal A_\pi(Q)=\mathcal B_{P_{Q,\pi}}^*Q$. The following lemma establishes convergence within Stage 1 and bounds how its limiting critic changes with the scoring policy.

For a fixed scoring policy $\pi$, Stage 1 iterates $Q\leftarrow\mathcal A_\pi(Q)$ to obtain $\bar Q(\pi)$ and $P_{\bar Q(\pi),\pi}$. Stage 2 takes the Gibbs policy under the resulting posterior. Thus a complete cycle sets $\bar Q_k=\bar Q(\pi_k)$, $\bar P_k=P_{\bar Q_k,\pi_k}$, and $\pi_{k+1}=\pi_{\bar Q_k}$. The lemma below establishes that Stage 1 is well-defined, and the theorem controls the complete cycle.

\paragraph{Joint Sensitivity of the Posterior}
For any $Q,U\in\mathcal Q$ and $\pi,\pi'\in\Pi$, the ordinary values satisfy $\|V_Q^\pi-V_U^{\pi'}\|_\infty\le\|Q-U\|_\infty+Q_{\max}d(\pi,\pi')$. Hence the difference between the corresponding residuals is at most $(1+\gamma)\|Q-U\|_\infty+\gamma Q_{\max}d(\pi,\pi')$. Since each residual has magnitude at most $D$, batch averaging gives
\begin{equation*}
\max_m|\mathcal F_{Q,\pi}(T_m;E)-\mathcal F_{U,\pi'}(T_m;E)|
\le 2D\big[(1+\gamma)\|Q-U\|_\infty+\gamma Q_{\max}d(\pi,\pi')\big].
\end{equation*}
Softmax is 1-Lipschitz from the sup norm on logits to the $\ell_1$ norm on probabilities. Indeed, for logits moving in direction $z$, its derivative has $\ell_1$ norm $\mathbb E_p|z-\mathbb E_pz|\le\sqrt{\operatorname{Var}_p(z)}\le\|z\|_\infty$; integration along the line segment proves the bound. Applying it to $\log P_0-\beta\mathcal F$ yields
\begin{equation*}
\|p(Q,\pi)-p(U,\pi')\|_1
\le 2\beta D\big[(1+\gamma)\|Q-U\|_\infty+\gamma Q_{\max}d(\pi,\pi')\big].
\end{equation*}
For any $W\in\mathcal Q$, center the targets $y_m(W;s,a)$ at the midpoint of their model range. The differences in model weights sum to zero, so
\begin{equation*}
\|\mathcal B_{P_{Q,\pi}}^*W-\mathcal B_{P_{U,\pi'}}^*W\|_\infty
\le\tfrac{\Delta}{2}\|p(Q,\pi)-p(U,\pi')\|_1.
\end{equation*}
Combining this estimate with fixed-posterior contraction gives, for $\mathcal A_\pi(Q)=\mathcal B_{P_{Q,\pi}}^*Q$,
\begin{equation*}
\|\mathcal A_\pi(Q)-\mathcal A_{\pi'}(U)\|_\infty
\le\kappa\|Q-U\|_\infty+L\,d(\pi,\pi'),
\qquad
\kappa=\gamma+\beta D\Delta(1+\gamma),\quad L=\beta D\Delta\gamma Q_{\max}.
\end{equation*}
Here $L$ denotes the sensitivity bound used only in this proof. Every $\mathcal A_\pi$ maps $\mathcal Q$ into itself because $\|\mathcal A_\pi(Q)\|_\infty\le R_{\max}+\gamma Q_{\max}\le Q_{\max}$.

\setcounter{lemma}{0}
\begin{lemma}[Convergence Within a Fixed-Policy Critic Phase]
\label{app:lem:critic_phase}
Suppose $\lambda<1$. For each fixed scoring policy $\pi$, the map $\mathcal G_n^\pi(Q)=(\mathcal B_{P_{Q,\pi}}^*)^nQ$ is a contraction on $\mathcal Q$ with modulus at most $\gamma^n+(1-\gamma^n)\lambda$ for every integer $n\ge1$. Its limit as $n\to\infty$ has modulus at most $\lambda$. These maps share a unique fixed point $\bar Q(\pi)$, and their iterates converge to this point even with varying $n$. Moreover, for any $\pi,\pi'\in\Pi$,
\begin{equation*}
\|\bar Q(\pi)-\bar Q(\pi')\|_\infty
\le
\frac{\beta D\Delta\gamma Q_{\max}}
{(1-\gamma)(1-\lambda)}
\,d(\pi,\pi').
\end{equation*}
\end{lemma}

\begin{proof}[Proof of Lemma~\ref{app:lem:critic_phase}]
The preceding joint bound with $\pi'=\pi$ makes $\mathcal A_\pi=\mathcal G_1^\pi$ a $\kappa$-contraction, so its fixed point $\bar Q(\pi)$ exists and is unique. For general $n$, let $X_0=Q$, $Z_0=U$, $X_{j+1}=\mathcal B_{P_{Q,\pi}}^*X_j$, and $Z_{j+1}=\mathcal B_{P_{U,\pi}}^*Z_j$. The posterior bound gives $\|X_{j+1}-Z_{j+1}\|_\infty\le\gamma\|X_j-Z_j\|_\infty+\beta D\Delta(1+\gamma)\|Q-U\|_\infty$. Summing the geometric series yields the stated modulus. Frozen-posterior contraction ensures that $\mathcal G_\infty^\pi$ exists, and taking $n\to\infty$ gives its modulus $\lambda$.

The equation $\bar Q(\pi)=\mathcal B_{P_{\bar Q(\pi),\pi}}^*\bar Q(\pi)$ makes this critic a fixed point of every $\mathcal G_n^\pi$, including $n=\infty$. Their contraction moduli are uniformly bounded by $\kappa<1$, which proves convergence under varying step counts. Finally, apply the joint bound to $Q=\bar Q(\pi)$ and $U=\bar Q(\pi')$. Rearranging $(1-\kappa)\|\bar Q(\pi)-\bar Q(\pi')\|_\infty\le Ld(\pi,\pi')$ proves the sensitivity estimate.
\end{proof}

\setcounter{theorem}{2}
\begin{theorem}[Alternating Posterior and Policy Convergence]
\label{app:thm:adaptive}
Under the setting above, suppose
\begin{equation*}
\frac{\beta D\Delta}{1-\gamma}
\left(1+\gamma+\frac{\gamma Q_{\max}}{\alpha}\right)<1.
\end{equation*}
Then the cycle map $\pi\mapsto\pi_{\bar Q(\pi)}$ is a contraction on $(\Pi,d)$ with modulus at most $\rho=\frac{\beta D\Delta\gamma Q_{\max}}{\alpha(1-\gamma)(1-\lambda)}<1$. From any $\pi_0\in\Pi$, the alternating procedure satisfies $d(\pi_k,\pi^\star)\le\rho^k d(\pi_0,\pi^\star)$, $\bar Q_k\to Q^\star$, and $\bar P_k\to P^\star$. Its limit is the unique triple in $\Pi\times\mathcal Q\times\Delta(\mathcal T)$ satisfying
\begin{equation*}
P^\star=P_{Q^\star,\pi^\star},
\qquad
Q^\star=\mathcal B_{P^\star}^*Q^\star,
\qquad
\pi^\star=\pi_{Q^\star}.
\end{equation*}
In particular, $\pi^\star$ maximizes $\widehat{\mathcal{J}}_{P^\star}^\alpha$.
\end{theorem}

\begin{proof}[Proof of Theorem~\ref{app:thm:adaptive}]
The assumed inequality is equivalent to $\lambda+L/[\alpha(1-\gamma)]<1$. It implies $\lambda<1$, so Lemma~\ref{app:lem:critic_phase} makes the critic--posterior phase well-defined for every $\pi\in\Pi$. Because $\bar Q(\pi)=\mathcal B_{P_{\bar Q(\pi),\pi}}^*\bar Q(\pi)$, Theorem~\ref{app:thm2} identifies $\pi_{\bar Q(\pi)}$ as a regularized optimal policy for the resulting posterior.

The Gibbs policy is a softmax of $\log\mu+Q/\alpha$. The same softmax bound used above gives $d(\pi_Q,\pi_U)\le\|Q-U\|_\infty/\alpha$. Therefore the cycle map $\Phi(\pi)=\pi_{\bar Q(\pi)}$ satisfies
\begin{equation*}
d(\Phi(\pi),\Phi(\pi'))
\le\frac{1}{\alpha}\|\bar Q(\pi)-\bar Q(\pi')\|_\infty
\le\rho\,d(\pi,\pi'),\qquad
\rho=\frac{L}{\alpha(1-\kappa)}
=\frac{\beta D\Delta\gamma Q_{\max}}{\alpha(1-\gamma)(1-\lambda)}<1.
\end{equation*}
The finite product of policy simplexes is complete under $d$, and $\Phi$ maps it into itself. The contraction theorem gives a unique $\pi^\star$ and the geometric bound in the statement. The sensitivity estimate then implies $\bar Q_k\to Q^\star:=\bar Q(\pi^\star)$. The posterior bound implies $\bar P_k\to P^\star:=P_{Q^\star,\pi^\star}$; explicitly,
\begin{equation*}
\|\bar P_k-P^\star\|_1
\le 2\beta D\left[\frac{(1+\gamma)L}{1-\kappa}+\gamma Q_{\max}\right]d(\pi_k,\pi^\star).
\end{equation*}
The three limiting identities follow from the definitions and $\Phi(\pi^\star)=\pi^\star$. Conversely, any triple satisfying those identities must have $Q=\bar Q(\pi)$ by uniqueness of the fixed-policy critic solution, and hence $\pi=\Phi(\pi)$. Uniqueness of the cycle fixed point identifies that triple with $(\pi^\star,Q^\star,P^\star)$. Theorem~\ref{app:thm2} gives the stated regularized optimality under $P^\star$.
\end{proof}

\paragraph{How the Two Stages Converge}
The theorem controls how changes propagate through the two stages. In Stage 1, changing the critic changes the consistency scores and hence the posterior weights. These weights determine the next critic update. For each fixed scoring policy $\pi$, the condition $\lambda<1$ ensures that this repeated update converges to a unique critic $\bar Q(\pi)$ and its corresponding posterior.

Stage 2 changes the scoring policy for the next cycle, which changes the consistency scores and the critic obtained from the next Stage 1 solve. Convergence within each Stage 1 solve therefore leaves a further question: do the policies produced across cycles also converge? The lemma bounds how much $\bar Q(\pi)$ changes when $\pi$ changes, and the Gibbs update bounds how much the next policy changes when $\bar Q(\pi)$ changes. Together, these bounds show that a complete cycle reduces the distance between any two input policies by a factor of at most $\rho$. The theorem's condition guarantees both $\lambda<1$ and $\rho<1$, so all initial policies lead to the same limit.

At this limit, reconstructing the posterior from $Q^\star$ and $\pi^\star$ returns $P^\star$, and optimizing under $P^\star$ returns the same critic and Gibbs policy. Thus $\pi^\star$ is regularized optimal under the posterior that it and $Q^\star$ induce. Appendix~\ref{appendix:counterexample} constructs a two-model example showing why posterior reconstruction requires an additional stability condition beyond $\gamma<1$.

\paragraph{Interpreting the Condition}
The residual bound satisfies $D\le2Q_{\max}$ uniformly over scoring policies. Model disagreement can be bounded directly from the kernels:
\begin{equation*}
\Delta\le\gamma Q_{\max}\max_{m,\ell,s,a}\|T_m(\cdot|s,a)-T_\ell(\cdot|s,a)\|_1.
\end{equation*} 
For $\beta=1$ and $\gamma=0.99$, the full alternating theorem requires $D\Delta < 0.01/(1.99 + 0.99 Q_{\max}/\alpha)$, whereas the fixed-policy critic lemma only requires $D\Delta < 0.01/1.99$. Thus, policy feedback introduces an additional stability requirement.

\subsection{Why Posterior Reconstruction Requires a Stability Condition}
\label{appendix:counterexample}
A fixed posterior gives a $\gamma$-contractive Bellman operator. Reconstructing that posterior from the critic can amplify differences between critics, even when the policy is fixed. The following finite MDP makes this effect explicit and shows that the adaptive operator can have multiple fixed points.

\paragraph{A Two-Model Construction}
Let $\mathcal S=\{s_0,s_+,s_-\}$ and let there be one action $a$, so the policy and reference policy are uniquely determined. Take rewards $r(s_0,a)=0$ and $r(s_\pm,a)=\pm(1-\gamma)$. Both models make $s_+$ and $s_-$ absorbing; they differ only at $s_0$: $T_+$ moves to $s_+$, whereas $T_-$ moves to $s_-$. Use equal prior weights and the fixed evidence batch $E=\{(s_0,a)\}$. Set $Q_{\max}=1$ and consider critics $Q_q$ with $Q_q(s_0,a)=q\in[-1,1]$ and $Q_q(s_\pm,a)=\pm1$. This family is invariant under the adaptive Bellman update, since the targets at the absorbing states remain $\pm1$. With one action, the ordinary and soft Bellman operators coincide for every $\alpha>0$.

\paragraph{The Adaptive Update}
At $s_0$, model $T_\pm$ has target $\pm\gamma$ and consistency score $\mathcal F_{Q_q,\pi}(T_\pm;E)=(q\mp\gamma)^2$. Writing $p_\pm(q)=P_{Q_q,\pi}(T_\pm)$, Equation~\eqref{app:model:eq1} gives
\begin{equation*}
p_+(q)=\frac{1}{1+e^{-4\beta\gamma q}},\qquad
p_-(q)=1-p_+(q),\qquad
(\mathcal A_\pi Q_q)(s_0,a)=f(q):=\gamma\tanh(2\beta\gamma q).
\end{equation*}
Thus $\mathcal A_\pi Q_q=Q_{f(q)}$. Since $\|Q_q-Q_u\|_\infty=|q-u|$, contraction of the full adaptive operator would imply contraction of $f$. However, $f'(0)=2\beta\gamma^2$. At $\beta=1$ and $\gamma=0.99$, this derivative is $1.9602>1$: critics sufficiently close to $Q_0$ can move farther apart after an update. In contrast, freezing any posterior makes the target at $s_0$ independent of $q$ on this family.

\paragraph{Multiple Fixed Points}
When $2\beta\gamma^2>1$, $q=0$ is a fixed point, and $f(q)-q>0$ for sufficiently small positive $q$. Since $f(\gamma)-\gamma<0$, continuity gives a second fixed point $q_+\in(0,\gamma)$. Odd symmetry gives a third, $-q_+$. Consequently, $Q_0$, $Q_{q_+}$, and $Q_{-q_+}$ are distinct fixed points of the full adaptive operator. This establishes failure of contraction and uniqueness; it does not establish divergence. Indeed, $f$ is increasing and maps $[-1,1]$ into itself, so each scalar iteration is monotone after its first step and bounded, and therefore converges to a fixed point.

\paragraph{Connection to the Convergence Guarantees}
A positive critic at $s_0$ gives the positive-target model a smaller consistency score and hence a larger posterior weight. The resulting target is more positive, reinforcing the initial difference. Lemma~\ref{app:lem:critic_phase} controls this effect through $\lambda<1$, while Theorem~\ref{app:thm:adaptive} also controls changes in the scoring policy across cycles. In this construction, $\Delta=2\gamma$ and $D\ge1+\gamma$, so $\lambda>1$ at the parameters above. This example illustrates why discounting alone is insufficient and motivates explicit control of posterior sensitivity. The bounds in our lemma and theorem are sufficient conditions.

\subsection{Stochastic Approximation and Inexact Updates}
\label{appendix:stochastic}
\setcounter{corollary}{2}
\begin{corollary}[Posterior-Sampled Tabular Updates]
\label{app:cor:stochastic}
Let $\mathcal A:\mathcal Q\to\mathcal Q$ be a fixed contraction in the sup norm. Let $\{\mathscr H_t\}$ be the increasing sequence of information available before sampling. Suppose $Q_t$ is determined by $\mathscr H_t$ and $Y_t$ is observed by $\mathscr H_{t+1}$, with $|Y_t(s,a)|\le Q_{\max}$ and $\mathbb E[Y_t|\mathscr H_t]=\mathcal A(Q_t)$. If every coordinate is updated by Equation~\eqref{app:thm:policy_eval}, $Q_0\in\mathcal Q$, and the deterministic step sizes satisfy $0<\eta_t\le1$, $\sum_t\eta_t=\infty$, and $\sum_t\eta_t^2<\infty$, then $Q_t$ converges almost surely to the unique fixed point of $\mathcal A$.
\end{corollary}

\begin{proof}[Proof of Corollary~\ref{app:cor:stochastic}]
Let $Q^\dagger$ be the fixed point and $\kappa<1$ a contraction modulus for $\mathcal A$. Set $\xi_{t+1}=Y_t-\mathcal A(Q_t)$. The conditional-mean assumption gives $\mathbb E[\xi_{t+1}|\mathscr H_t]=0$. The target and iterate bounds imply uniformly bounded conditional second moments. Since $\eta_t\in(0,1]$, the update is a convex combination of vectors in $\mathcal Q$, so the iterates remain in $\mathcal Q$.

Define $z_0=0$ and $z_{t+1}=(1-\eta_t)z_t+\eta_t\xi_{t+1}$. The martingale $M_t=\sum_{j<t}\eta_j\xi_{j+1}$ is bounded in $L_2$ because $\sum_j\eta_j^2<\infty$ and the conditional second moments are bounded. It therefore converges almost surely. To see that $z_t\to0$ on this event, fix $t_0$ and expand
\begin{equation*}
z_t=\left[\prod_{j=t_0}^{t-1}(1-\eta_j)\right]z_{t_0}
+\sum_{i=t_0}^{t-1}\left[\prod_{j=i+1}^{t-1}(1-\eta_j)\right](M_{i+1}-M_i).
\end{equation*}
The first term tends to zero because $\sum_j\eta_j=\infty$. Summation by parts bounds the norm of the second term by $2\sup_{j\ge t_0}\|M_j-M_{t_0}\|_\infty$, since its weights increase from a value in $[0,1]$ to 1. Letting $t_0\to\infty$ proves $z_t\to0$.

Now set $u_t=Q_t-Q^\dagger-z_t$. Subtracting the noise recursion gives
\begin{equation*}
\begin{split}
\|u_{t+1}\|_\infty
&\le(1-\eta_t)\|u_t\|_\infty
+\eta_t\kappa\|u_t+z_t\|_\infty\\
&\le[1-(1-\kappa)\eta_t]\|u_t\|_\infty
+\eta_t\kappa\|z_t\|_\infty.
\end{split}
\end{equation*}
For any $\delta>0$, eventually $\kappa\|z_t\|_\infty\le(1-\kappa)\delta$. The last recursion then bounds the positive part of $\|u_t\|_\infty-\delta$ by a product tending to zero. Hence $\limsup_t\|u_t\|_\infty\le\delta$. Since $\delta$ is arbitrary, $u_t\to0$ and therefore $Q_t\to Q^\dagger$ almost surely.
\end{proof}

\setcounter{corollary}{1}
\begin{corollary}[Neural Approximation of the Alternating Updates]
\label{app:cor:neural_alternating}
Under the assumptions of Theorem~\ref{app:thm:adaptive}, let $\pi_k=\pi_{\phi_k}$ and let $Q_{\theta_k}\in\mathcal Q$ be the neural critic returned by Stage 1. Reconstruct $P_k=P_{Q_{\theta_k},\pi_k}$ and define the critic and policy errors as $u_k=\|Q_{\theta_k}-\bar Q(\pi_k)\|_\infty$ and $v_k=d(\pi_{\phi_{k+1}},\pi_{Q_{\theta_k}})$. Then
\begin{equation*}
d(\pi_{k+1},\pi^\star)\le\rho\,d(\pi_k,\pi^\star)+u_k/\alpha+v_k.
\end{equation*}
If $u_k,v_k\to0$, then $\pi_k\to\pi^\star$ in $d$, $Q_{\theta_k}\to Q^\star$ in the sup norm, and $P_k\to P^\star$ in $\ell_1$. If instead $u_k/\alpha+v_k\le\bar e$ for all $k$, then $$\limsup_k d(\pi_k,\pi^\star)\le\bar e/(1-\rho).$$
\end{corollary}

\begin{proof}[Proof of Corollary~\ref{app:cor:neural_alternating}]
For neural updates, take $Q_{\theta_k}=Q_{\theta_k}$ and $\pi_k=\pi_{\phi_k}$ in the following error bound. More generally, under Theorem~\ref{app:thm:adaptive}, suppose cycle $k$ returns $Q_{\theta_k}\in\mathcal Q$ with $\|Q_{\theta_k}-\bar Q(\pi_k)\|_\infty\le u_k$, reconstructs $\widetilde P_k=P_{Q_{\theta_k},\pi_k}$, and chooses $\pi_{k+1}\in\Pi$ with $d(\pi_{k+1},\pi_{Q_{\theta_k}})\le v_k$. The Gibbs sensitivity bound and cycle contraction imply
\begin{equation*}
d(\pi_{k+1},\pi^\star)
\le\rho\,d(\pi_k,\pi^\star)+u_k/\alpha+v_k.
\end{equation*}
If $u_k,v_k\to0$, unrolling this recursion gives $\pi_k\to\pi^\star$: its initial-error term is multiplied by $\rho^k$, and the convolution of the vanishing errors with the summable geometric weights tends to zero. The critic sensitivity and posterior bounds then give $Q_{\theta_k}\to Q^\star$ and $\widetilde P_k\to P^\star$. If instead $u_k/\alpha+v_k\le\bar e$, the same recursion gives $\limsup_k d(\pi_k,\pi^\star)\le\bar e/(1-\rho)$.

\end{proof}

\paragraph{Stopping Criteria and Interpretation}
A critic stopping criterion can use the adaptive residual: $\|Q_{\theta_k}-\mathcal A_{\pi_k}(Q_{\theta_k})\|_\infty\le\xi_k$ implies $u_k\le\xi_k/(1-\kappa)$ by contraction and the fixed-point equation. Thus vanishing adaptive residuals and vanishing policy errors are sufficient for convergence. These bounds measure error relative to the exact Gibbs cycle; a single KL-constrained step, a neural optimizer, or a moving target network must satisfy the stated error controls for this conclusion to apply.

The errors $u_k$ and $v_k$ include both function approximation and incomplete optimization. Their required bounds are uniform over the finite state-action space; a small empirical training loss alone is insufficient. The conclusions concern network outputs and do not require convergence of the network parameters. Continuous state or action spaces require a separate extension of the function spaces and uniform bounds.

\subsection{A Finite Policy Step under a Fixed Posterior}
\label{appendix:thm:improve}
\setcounter{theorem}{1}
\begin{theorem}[Local Policy Improvement]
\label{app:thm3}
Fix $P_k$ and suppose $\pi_k$ has full support. In independent tabular-policy coordinates $\theta$, let $g_k=\nabla \widehat{\mathcal{J}}_{P_k}^\alpha(\theta_k)\ne0$ and let $F_k\succ0$ be the Hessian of the policy KL at $\theta_k$. Suppose both objectives are twice continuously differentiable locally and $\nabla \mathcal{J}_{P_k}(\theta_k)^TF_k^{-1}g_k>0$. For any exact solution $\theta_k+\delta_\epsilon$ of Equation~\eqref{app:opt:eq3}, as $\epsilon\rightarrow0$,
\begin{equation*}
\begin{aligned}
&\mathcal{J}_{P_k}(\theta_k+\delta_\epsilon)-\mathcal{J}_{P_k}(\theta_k)\\
&\quad=\sqrt{\frac{2\epsilon}{g_k^TF_k^{-1}g_k}}
\nabla \mathcal{J}_{P_k}(\theta_k)^TF_k^{-1}g_k+o(\sqrt\epsilon).
\end{aligned}
\end{equation*}
Consequently, every exact solution of Equation~\eqref{app:opt:eq3} improves $\mathcal{J}_{P_k}$ for all sufficiently small $\epsilon>0$.
\end{theorem}

\begin{proof}[Proof of Theorem~\ref{app:thm3}]
Throughout the proof, $P_k$ and $\mu$ are fixed. Use independent probability coordinates for the tabular policy and write $\delta=\theta-\theta_k$. Because $\nu$ is positive on every state, the policy KL divergence is zero only at $\pi_k$. Compactness of the product of policy simplexes implies that all policies with sufficiently small KL divergence lie in a neighborhood of $\pi_k$. Full support makes this neighborhood an interior coordinate chart.

Taylor expansion in this chart gives
\begin{equation*}
D_{\mathrm{KL},\nu}(\pi_{\theta_k+\delta}\Vert\pi_k)
=\tfrac12\delta^TF_k\delta+o(\|\delta\|^2),
\qquad
\widehat{\mathcal{J}}_{P_k}^\alpha(\theta_k+\delta)
=\widehat{\mathcal{J}}_{P_k}^\alpha(\theta_k)+g_k^T\delta+O(\|\delta\|^2).
\end{equation*}
Since $F_k\succ0$, a feasible step satisfies $\delta=O(\sqrt\epsilon)$. Let $\delta_\epsilon$ be any exact maximizing step and write $u_\epsilon=\delta_\epsilon/\sqrt{2\epsilon}$. Every limit point satisfies $u^TF_k u\le1$. The unique maximizer of $g_k^Tu$ on this ellipsoid is
\begin{equation*}
u^*=\frac{F_k^{-1}g_k}{\sqrt{g_k^TF_k^{-1}g_k}}.
\end{equation*}
For each fixed $t\in(0,1)$, the candidate step $\sqrt{2\epsilon}(1-t)u^*$ is feasible for all sufficiently small $\epsilon$. Optimality of $\delta_\epsilon$, followed by $\epsilon\to0$ and then $t\to0$, implies that every limit point of $u_\epsilon$ maximizes $g_k^Tu$ on the ellipsoid. Uniqueness therefore gives
\begin{equation*}
\delta_\epsilon
=\sqrt{\frac{2\epsilon}{g_k^TF_k^{-1}g_k}}F_k^{-1}g_k+o(\sqrt\epsilon).
\end{equation*}
Expanding the unregularized objective yields
\begin{equation*}
\mathcal{J}_{P_k}(\theta_k+\delta_\epsilon)-\mathcal{J}_{P_k}(\theta_k)
=\sqrt{\frac{2\epsilon}{g_k^TF_k^{-1}g_k}}
\nabla \mathcal{J}_{P_k}(\theta_k)^TF_k^{-1}g_k+o(\sqrt\epsilon).
\end{equation*}
The leading coefficient is strictly positive by assumption. The entire difference is therefore positive for all sufficiently small $\epsilon>0$. This proves a finite-step improvement and does not require discarding an uncontrolled Taylor remainder.
\end{proof}

\subsection{Local Improvement for Neural Policies}
\label{appendix:cor:neural}
\setcounter{corollary}{0}
\begin{corollary}[Local Improvement for Neural Policies]
\label{app:cor:neural}
Fix $P_k,\mu,\alpha,\nu$, and let $\pi_\phi$ be a neural policy with $\pi_{\phi_k}=\pi_k$. Suppose the policy has full support locally, and the policy map, both return objectives, and $h_k(\phi)=D_{\mathrm{KL},\nu}(\pi_\phi\Vert\pi_k)$ are twice continuously differentiable near $\phi_k$. In these parameters, define $g_k=\nabla_\phi\widehat{\mathcal{J}}_{P_k}^\alpha(\pi_\phi)|_{\phi_k}\ne0$ and $F_k=\nabla_\phi^2h_k(\phi_k)\succ0$. Assume $\nabla_\phi\mathcal{J}_{P_k}(\pi_\phi)|_{\phi_k}^{T}F_k^{-1}g_k>0$. For any sufficiently small fixed $r>0$, restrict Equation~\eqref{app:opt:eq3} to policies $\pi_\phi$ with $\|\phi-\phi_k\|_2\le r$. Every exact maximizer satisfies the return expansion in Theorem~\ref{app:thm3}, with derivatives taken in $\phi$, and improves $\mathcal{J}_{P_k}$ for all sufficiently small $\epsilon>0$.
\end{corollary}

\begin{proof}[Proof of Corollary~\ref{app:cor:neural}]
Write $\delta=\phi-\phi_k$, $f(\phi)=\widehat{\mathcal{J}}_{P_k}^\alpha(\pi_\phi)$, and $j(\phi)=\mathcal{J}_{P_k}(\pi_\phi)$. Since $h_k\ge0$ and $h_k(\phi_k)=0$, its gradient vanishes at $\phi_k$. The smoothness and positive-definiteness assumptions give $h_k(\phi_k+\delta)=\tfrac12\delta^TF_k\delta+o(\|\delta\|_2^2)$. Consequently, there exist $r_0,c>0$ such that $h_k(\phi_k+\delta)\ge c\|\delta\|_2^2$ whenever $\|\delta\|_2\le r_0$; choose $r_0$ so that its closed ball lies in the smoothness neighborhood.

Fix $0<r\le r_0$. The feasible set $\{\phi:\|\phi-\phi_k\|_2\le r,\ h_k(\phi)\le\epsilon\}$ is nonempty and compact, so an exact maximizer exists. Every feasible step satisfies $\|\delta\|_2\le\sqrt{\epsilon/c}$. For $\epsilon<cr^2$, the parameter-ball constraint is therefore inactive. Taylor expansion yields $f(\phi_k+\delta)=f(\phi_k)+g_k^T\delta+O(\|\delta\|_2^2)$. For any maximizing step $\delta_\epsilon$, the scaled sequence $\delta_\epsilon/\sqrt{2\epsilon}$ is bounded, and its limit points satisfy $u^TF_ku\le1$.

Let $u^*=F_k^{-1}g_k/\sqrt{g_k^TF_k^{-1}g_k}$. For every fixed $t\in(0,1)$, the step $\sqrt{2\epsilon}(1-t)u^*$ satisfies both constraints for sufficiently small $\epsilon$. Comparing an exact maximizer with these candidates and then taking $\epsilon\downarrow0$ and $t\downarrow0$ shows that each limit point maximizes $g_k^Tu$ over $u^TF_ku\le1$. This maximizer is uniquely $u^*$, hence $\delta_\epsilon=\sqrt{2\epsilon}\,u^*+o(\sqrt\epsilon)$. Expanding the original return gives
\begin{equation*}
j(\phi_k+\delta_\epsilon)-j(\phi_k)
=\sqrt{\frac{2\epsilon}{g_k^TF_k^{-1}g_k}}\nabla j(\phi_k)^TF_k^{-1}g_k+o(\sqrt\epsilon).
\end{equation*}
The alignment assumption makes the leading coefficient positive. This conclusion holds for every selection of exact maximizers: a sequence of non-improving maximizers with $\epsilon\downarrow0$ would contradict the same expansion. Thus all exact maximizers improve the original return for sufficiently small $\epsilon$.
\end{proof}

In the finite-space setting, the KL Hessian also has the expression $F_k=\mathbb E_{s\sim\nu,\,a\sim\pi_k}[\ell_k(s,a)\ell_k(s,a)^T]$, where $\ell_k(s,a)=\nabla_\phi\log\pi_\phi(a|s)|_{\phi_k}$. Therefore $F_k\succ0$ exactly when no nonzero parameter direction has zero directional score at every state-action pair. This condition can be imposed on a chosen set of trainable coordinates, with the remaining parameters fixed. The result concerns an exact local constrained update; using a penalty or an approximate optimizer requires a separate bound on the optimization error.

\section{Experimental Details}
\label{appendix:experiment}
\subsection{Dataset}
\subsubsection{D4RL}
To comprehensively assess performance across diverse offline scenarios, our experimental assessment encompasses a comprehensive benchmark spanning eighteen distinct experimental settings. These domains arise from the combination of three continuous control tasks (hopper, walker2d, and halfcheetah) with six diverse offline datasets that vary in quality and collection strategy.

\begin{itemize} \item \textbf{Random:} A dataset generated by a randomly initialized policy. This setting represents the most challenging scenario with low-quality data and high entropy, testing the algorithm's ability to extract meaningful learning signals from noisy interactions without any expert guidance.

	\item \textbf{Expert:} A dataset collected by a fully converged Soft Actor-Critic (SAC) agent. This dataset consists of high-return trajectories but covers a relatively narrow region of the state-action space. It evaluates the algorithm's capacity for imitation and its ability to avoid overfitting to a deterministic behavior policy.

	\item \textbf{Medium:} A dataset generated by a policy trained to approximately one-third of the expert's performance. This represents a competent but suboptimal baseline, challenging the algorithm to extrapolate beyond the demonstrated performance to achieve optimality.

	\item \textbf{Medium-Expert:} A composite dataset constructed by mixing equal parts (50-50 split) of medium and expert trajectories. This setting introduces a multimodal data distribution with mixed quality, requiring the algorithm to effectively distinguish between optimal and suboptimal transitions and selectively leverage high-reward data while ignoring inferior samples.

	\item \textbf{Medium-Replay:} A dataset comprising the entire replay buffer recorded during the training of the medium policy. Unlike the fixed medium dataset, this captures the full evolution of the agent from random initialization to intermediate competence, offering high state-space coverage but significant variance in transition quality.

	\item \textbf{Full-Replay:} This dataset encompasses the complete spectrum of policy evolution, ranging from random exploration to convergence. It offers the highest information density but also presents a challenge due to the extensive diversity and scale of the data.
\end{itemize}

\subsubsection{Optimal Liquidation}
\label{appendix:liquidation}
This benchmark is an adaptation of the Optimal Liquidation Problem for offline RL, adhering to the protocol established by \cite{rigter2023one}. The primary objective is to maximize the realized revenue obtained from liquidating an initial inventory of 100 units of a specific asset (Currency A) into a target asset (Currency B) within a finite time horizon $T$, governed by stochastically evolving market dynamics. At every time step, the agent serves as an execution algorithm that must determine the optimal liquidation schedule, specifically deciding the proportion of the remaining inventory to execute.

This formulation presents a fundamental sequential decision-making challenge characterized by the tension between market risk and potential return. At each decision point, the agent must strategically balance the incentive to delay execution, which offers the possibility of capturing favorable price movements, against the substantial risk of adverse market shifts that could significantly erode the asset's value. Consequently, the optimal policy must transcend simple heuristics to navigate a complex landscape of price uncertainty. It requires distinguishing between conservative strategies, such as immediate liquidation or uniform execution (akin to Time-Weighted Average Price strategies) which prioritize variance reduction, and aggressive strategies that hold inventory in anticipation of price recovery despite the exposure to heightened volatility.

\paragraph{Formulation:} The liquidation environment is mathematically formulated as a finite-horizon Markov Decision Process. The state space is characterized by a three-dimensional vector $s_t=(t, m_t, p_t)$, encompassing the current discrete time step $t \in \{0, 1, \dots, T\}$, the remaining inventory of Currency A denoted by $m_t \in [0, 100]$, and the instantaneous exchange rate $p_t$. The episode initializes with an exchange rate sampled from a Gaussian distribution $p_0 \sim \mathcal{N}(1, 0.05^2)$. The action space is continuous and one-dimensional, defined as $\mathcal{A} = [-1, 1]$. The interpretation of an action $a_t \in \mathcal{A}$ is twofold: a positive value $a_t \in (0, 1]$ corresponds to the decision to convert a proportion $a_t$ of the current inventory $m_t$ into Currency B; conversely, a non-positive action $a_t \in [-1, 0]$ signifies a ``hold" decision where no currency is converted at the current step. The objective is to maximize the cumulative reward, where the reward function is defined as the amount of Currency B obtained at each step. The system dynamics are primarily driven by the stochastic evolution of the exchange rate. To realistically capture mean-reverting market properties, $p_t$ is modeled using an Ornstein-Uhlenbeck (OU) process governed by the following stochastic differential equation:
\begin{equation*}
	dp_t = \theta(\mu - p_t)dt + \sigma dW_t,
\end{equation*}
Here $W_t$ is a standard Wiener process. In this environment-specific equation, $\theta$ is the mean-reversion rate and $\mu$ is the equilibrium price; these are distinct from the policy notation used in the analysis. In the specific implementation, the process parameters are configured as follows: the rate of mean reversion is set to $\theta = 0.05$, the long-term equilibrium price is $\mu = 1.5$, and the volatility coefficient is $\sigma = 0.2$.
\paragraph{Dataset} The dataset is generated by a random behavioral policy. This policy adopts a mixture strategy at each time step. Specifically, with a probability of $0.8$, the policy executes a non-conversion action (effectively sampling a negative action), opting to hold the current inventory. Conversely, with a probability of $0.2$, it selects a conversion action. In this latter case, the specific proportion of currency to be converted is determined by sampling uniformly from the feasible positive action space, thereby introducing stochasticity.

\subsubsection{Reference Performance}
The normalized score benchmarks an algorithm's performance on a standardized scale. On this scale, a random policy is set to 0 and an expert policy is set to 100. A score greater than 100 thus signifies that the policy learned from the offline dataset outperforms the online-trained expert policy. The score is computed as follows:

\begin{equation}
	\mathrm{score}_{\mathrm{algo}}=100\times\frac{\mathrm{performance}_{\mathrm{algo}}-\mathrm{performance}_{\mathrm{random}}}   {\mathrm{performance}_{\mathrm{expert}}-\mathrm{performance}_{\mathrm{random}}}.
\end{equation}
The reference performance is reported in Table~\ref{appendix:tab1}.
\begin{table}[htbp]
	\centering
	\caption{Reference performance across different environments.}
	\label{appendix:tab1}
	\begin{tabular}{c c c}
		\toprule
		\textbf{Environment} & \textbf{Random Policy} & \textbf{Expert Policy} \\
		\midrule
		Halfcheetah & -280.18 & 12135.0 \\
		Hopper & -20.27 & 3234.3 \\
		Walker2d & 1.63 & 4592.3 \\
		Offline optimal liquidation & 0.0 & 135.0 \\
		\bottomrule
	\end{tabular}
\end{table}
\subsection{Hyperparameters} Table~\ref{tab:hyperparameters} presents the detailed hyperparameter configurations. For the simpler hopper environment, dynamics models with layer and batch sizes of 256 units are sufficient. However, the increased complexity of walker2d and halfcheetah uses larger architectures with 512-unit layers in the reported configuration.
\begin{table}[htbp]
	\centering
	\caption{Hyperparameters of PSPO.}
	\label{tab:hyperparameters}
	\begin{tabular}{l|c}
		\toprule
		\textbf{Parameter} & \textbf{Value} \\
		\midrule
		Dynamics model learning rate & $10^{-4}$ \\
		Policy learning rate & $3 \cdot 10^{-5}$ \\
		Critic (Q-value) learning rate & $3 \cdot 10^{-4}$ \\
		discounted factor ($\gamma$) & 0.99 \\
		Size of the model ensemble ($N$) & 10 \\
		Likelihood Coefficient ($\beta$) & 1.0 \\
		Layer size of policy & 256 \\
		Layer size of dynamics model & 512 \\
		Batch size for dynamics model learning & 512 \\
		Batch size for policy learning & 256 \\
		Dynamics model training epochs & 1000 \\
		\bottomrule
	\end{tabular}
\end{table}
\subsection{Dynamics Model}
\label{appendix:exp_dynamics_model}
Following standard model-based RL protocols, we parameterize the transition dynamics using neural networks. Each network takes a state-action pair as input and predicts a Gaussian joint density $p_{\psi_i}(s',r'\mid s,a)$ with diagonal covariance. We write $T_{\psi_i}(s'\mid s,a)$ for its next-state marginal. The Bellman analysis conditions on a shared bounded reward function $r$, as specified in Section~\ref{app:setting}. To effectively capture epistemic uncertainty, we construct a large candidate pool of 100 models, each trained independently via Maximum Likelihood Estimation (MLE) on the offline dataset. For model-generated rollouts during training, an active ensemble of $N$ models is sampled from this pre-trained pool. At deployment, action selection uses only the policy network. The convergence analysis conditions on a fixed active ensemble.
\subsection{Belief Distribution}
We construct a dynamics model ensemble and maintain a belief distribution over it. While the prior distribution $P_0(T)$ over model ensemble can be initialized in various ways to encode specific environmental prior knowledge, practical deployment environments often lack exhaustive initial specifications. A uniform prior assigns equal probabilities to ensemble members before reweighting \cite{chua2018deep}.
\subsection{Compute Infrastructure}
All computational experiments and benchmarks reported in this study were executed on a high-performance computing workstation designed for large-scale deep learning tasks. The hardware configuration includes: \begin{itemize}
	\item GPU Acceleration: A cluster of 8 $\times$ NVIDIA GeForce RTX 3090 Ti GPUs, utilized for parallelized ensemble training and policy optimization.
	\item Processing Unit: An Intel(R) Xeon(R) Platinum 8383C CPU @ 2.70GHz, ensuring efficient data preprocessing and environment simulation.
	\item Memory: 256 GB of system RAM to accommodate large-scale offline datasets (e.g., D4RL)..
\end{itemize}
The software environment consists of Ubuntu 20.04 LTS.

\section{Additional Discussion}
PSPO and Trust Region Policy Optimization (TRPO) \cite{pmlr-v37-schulman15} both constrain policy changes during optimization. PSPO additionally rebuilds the distribution used to select dynamics models. This creates a change in the Bellman operator that must be controlled separately from the policy step. Theorem~\ref{app:thm:adaptive} controls the full alternating cycle in the finite-space setting when each phase is solved exactly. Theorem~\ref{app:thm3} establishes local improvement of the original return for the separate KL-constrained policy step at a fixed posterior.

Posterior weighting and risk-sensitive objectives address different properties of a decision problem. Prior work studies model uncertainty and risk-sensitive RL using objectives and constraints tailored to those properties \cite{rigter2022rambo,rigter2023one,yang2026proactive,li2026tw}. PSPO optimizes a regularized expected return and uses a consistency-based model weighting. Its uncertainty diagnostics do not by themselves establish a risk-averse objective, a safety guarantee, or monotonic improvement of the true environment return.

\section{Limitations}
The alternating convergence theorem assumes uniform bounds over tabular critics and scoring policies, a fixed evidence batch, and complete critic--posterior and Gibbs policy solves. These conditions have not been established for the neural experiments. The practical algorithm uses finite KL-constrained policy updates, approximate critics, and moving target networks. The convergence guarantee in Appendix~\ref{appendix:stochastic} requires their errors relative to the exact alternating cycle to vanish. Full belief-state planning and convergence with an independently changing reference policy are outside the proved setting. The model-ensemble support is also fixed in these results; changing the active ensemble requires additional control.

\section{Broader Impact}
PSPO studies policy learning from offline data, with potential applications in robotics, industrial control, and resource management where online exploration is costly. Its practical impact depends on the quality of the data and learned dynamics. These applications motivate further evaluation of model errors and policy behavior under distribution shift.

%\section*{Acknowledgments}
%\label{sec:ack}
%We thank xxx for their contributions and support for the project.

\end{document}